%% file: main.tex
\documentclass{article}
\usepackage{iclr2026_conference,times}

\input{math_commands.tex}

\usepackage{booktabs}
\usepackage{url}
\usepackage{microtype}
\usepackage[table,rgb]{xcolor}
\definecolor{darkblue}{rgb}{0, 0, 0.5}
\usepackage[colorlinks=true, citecolor=darkblue, linkcolor=darkblue, urlcolor=darkblue]{hyperref}
\usepackage{amsmath}
\usepackage{bm}
\usepackage[nameinlink,capitalize,noabbrev]{cleveref}
\usepackage{graphicx}
\usepackage{inconsolata}
\usepackage{nicefrac}
\usepackage{booktabs}
\usepackage{enumitem}
\usepackage{amsthm}
\usepackage[symbol]{footmisc}
\usepackage{multirow}

\newtheorem{theorem}{Theorem}

\newtheorem{fact}{Fact}
\theoremstyle{definition}
\newtheorem{definition}{Definition}

\newtheorem{corollary}{Corollary}[theorem]

\newcommand*{\defeq}{\stackrel{\scalebox{0.5}{\text{def}}}{=}}

\definecolor{Red}{rgb}{0.8,0.0,0.05}
\definecolor{Green}{rgb}{0,0.6,0}
\definecolor{Blue}{rgb}{0.1, 0.1, 0.9}
\definecolor{LightBlue}{rgb}{0.9, 0.8, 1.0}
\definecolor{Grey}{rgb}{0.6,0.6,0.6}

\title{Dual-objective Language Models:\\
Training Efficiency Without Overfitting}

\author{David Samuel\footnotemark\\
University of Oslo\\
\texttt{davisamu@uio.no}
\And
\textbf{Lucas Georges Gabriel Charpentier}\footnotemark[1]\\
National Library of Norway\\\texttt{lucas.charpentier@nb.no}
}

\usepackage[most]{tcolorbox}

\newtcbtheorem[]{remark}{Remark}{
  colback=black!3!white,
  colframe=black!12!white,
  boxrule=0.5pt,
  boxrule=0mm,
  colbacktitle=black!12!white, 
  coltitle=black,
  sharp corners,
  left=2mm,
  leftrule=0mm,
  right=2mm,
  title={#1},
  toptitle=0.75mm, 
  fonttitle=\bfseries,
  description font=\mdseries\itshape,
  separator sign none,
  description delimiters none,
}{rem}

\iclrfinalcopy 
\begin{document}

\maketitle

\renewcommand*{\thefootnote}{\fnsymbol{footnote}}

\footnotetext[1]{Equal contribution. Work done at the Language Technology Group, University of Oslo.}


\renewcommand*{\thefootnote}{\arabic{footnote}}

\begin{abstract}
This paper combines autoregressive and masked-diffusion training objectives without any architectural modifications, resulting in flexible language models that outperform single-objective models. Autoregressive modeling has been a popular approach, partly because of its training efficiency; however, that comes at the cost of sensitivity to overfitting. On the other hand, masked-diffusion models are less efficient to train while being more resilient to overfitting. In this work, we demonstrate that dual-objective training achieves the best of both worlds. To derive the optimal balance between both objectives, we train and evaluate 50 language models under varying levels of data repetition. We show that it is optimal to combine both objectives under all evaluated settings and that the optimal balance is similar whether targeting autoregressive or masked-diffusion downstream performance.
\end{abstract}

\section{Introduction}

The dominant paradigm for training recent language models has been \textit{autoregressive} next-token prediction \citep{NEURIPS2020_1457c0d6}. This approach is remarkably efficient in training, allowing models to quickly absorb vast amounts of text. However, this comes with a significant drawback: a tendency to overfit when training data is repeated \citep{NEURIPS2023_9d89448b}. This issue is becoming increasingly critical as the community reaches the so-called \textit{data wall} -- the imminent exhaustion of available data required to train ever-larger models according to established scaling laws \citep{10.5555/3692070.3694094}.

An alternative approach, \textit{masked-diffusion} language modeling, offers a compelling solution to the overfitting problem  \citep{prabhudesai2025diffusionbeatsautoregressivedataconstrained, ni2025difflm}. Yet, this robustness comes at the cost -- these models are known to be much less sample-efficient than their autoregressive counterparts \citep{nie2025scaling}. The complementary strengths of these two approaches suggest combining them as a natural solution to counteracting their failure modes.

\begin{figure}[h!]
    \centering
    \includegraphics[width=\linewidth]{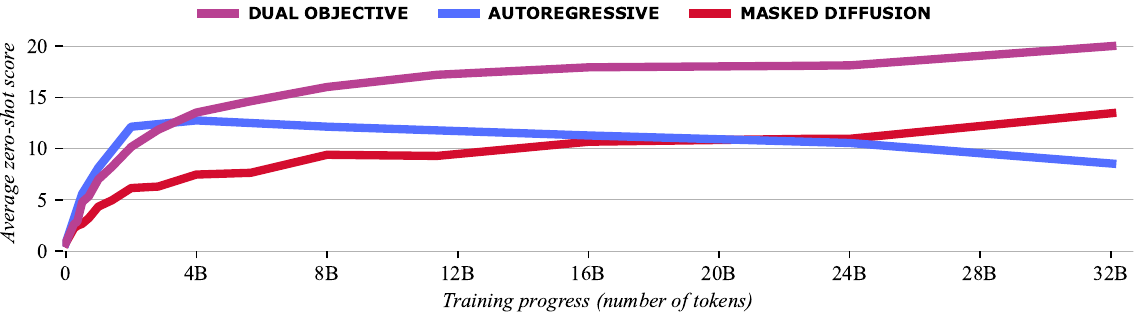}
    \caption{\textbf{The dynamics of zero-shot performance.} The three models are trained in a rather extreme setting -- 128 repetitions of the training corpus. The autoregressive objective (blue line) converges the fastest but also very quickly overfits; the masked-diffusion objective (red line) converges slowly but without being negatively affected by the high amount of repetitions. Combining both objectives together (purple line) results in fast convergence as well as robustness to overfitting.}
    \label{fig:motivation}
\end{figure}

In this work, we show that it is possible to achieve the best of both worlds by simultaneously training a single language model on both autoregressive and masked-diffusion objectives. The core idea is to use the training efficiency of the autoregressive objective for rapid initial learning while using the masked-diffusion objective to regularize the model and prevent it from overfitting. The effectiveness of this dual-objective approach is illustrated in \Cref{fig:motivation}. In the extreme data-constrained setting with 128 data repetitions, the purely autoregressive model learns quickly but then catastrophically overfits. The masked-diffusion model is immune to overfitting but converges very slowly. Our proposed dual-objective model combines the strengths of both and successfully leverages the given compute and data. The resulting models can be deployed as a standard autoregressive models with no inference overhead.

Building on this observation, we conduct a large-scale systematic study to find the optimal balance between these two objectives under varying degrees of data constraints. Our main contributions are:
\begin{itemize}[align=parleft,left=1em..2em]
    \item We propose a dual-objective training method that combines autoregressive and masked-diffusion losses, enabling a single model to excel at both unidirectional and bidirectional tasks.
    \item Through an extensive empirical study, we systematically map the relationship between data repetition, the ratio of training objectives, and final downstream performance. Demonstrating that our dual-objective approach is superior to single-objective training in all evaluated settings, for both autoregressive and masked-diffusion evaluation -- including the finding that dual-objective models outperform pure masked-diffusion models even in regular data settings.
    \item We derive two practical recommendations for setting the optimal objective ratio when training in both regular and data-constrained regimes, providing concrete guidelines for future training of large language models.
\end{itemize}

\section{Background}

Before diving into details of combining autoregressive and masked-diffusion models, we need to briefly describe those two modeling approaches and language modeling in general. As the name suggests, \emph{language models} are statistical models $p_{\bm{\theta}}(\cdot)$ of the true language distribution of some training corpus $\mathcal{D}$. The training corpus consists of sequences $\bm{x}=(x_1, x_2, \dots x_N)\in\mathcal{D}$ of subword tokens. The language models are trained by finding such parameters $\bm{\theta}$ that maximize the likelihood estimation \citep[MLE;][]{fisher1922mathematical, fisher1925theory}:
\begin{equation}
     \mathop{\text{argmax}}_{\bm{\theta}}\mathop{\mathbb{E}}_{\raisebox{-0.3pt}{$_{\bm{x}\, \sim\,\mathcal{D}}$}} \Bigl[\log p_{\bm{\theta}}(\bm{x}) \Bigr].
\label{eq:mle}
\end{equation}
In this paper, we combine two popular approaches for computing $p_{\bm{\theta}}(\cdot)$, \textit{autoregressive language modeling} and \textit{masked-diffusion language modeling}.

\subsection{Autoregressive language modeling}

Language models have a long tradition and since their inception in the seminal paper by \cite{shannon1951prediction}, they have been factored into a chain of next-token prediction terms $p_{\bm{\theta}}(x_i\,|\,\bm{x}_{<i})$:
\begin{equation}
    -\log p_{\bm{\theta}}(\bm{x}) = -\sum_{i=1}^{|\bm{x}|}\log p_{\bm{\theta}}(x_i \mid \bm{x}_{<i}) \defeq {\color{Blue}\mathcal{L}_{\text{AR}}(}\bm{x}{\color{Blue};}\,\bm{\theta}{\color{Blue})}.
\label{eq:ar}
\end{equation}
Computation of the next-token likelihoods can be efficiently parallelized when modeled by transformer networks \citep{NIPS2017_3f5ee243}, and thanks to their scalability, it has been the most popular paradigm behind the recent era of large language models \citep{NEURIPS2020_1457c0d6}.

\subsection{Masked-diffusion language modeling}

Masked-diffusion language models have recently become a popular alternative to autoregressive models \citep{NEURIPS2021_958c5305, 10.5555/3692070.3693403, 10.5555/3737916.3742051, ou2025your, nie2025large}. Computing $p_{\bm{\theta}}(\cdot)$ with masked-diffusion is slightly more complicated than with autoregression, but the resulting language model learns to handle full \textit{bidirectional} context, which can lead to increased performance on downstream tasks \citep{berglund2024the, 10.5555/3737916.3738000}.

First, following \cite{NEURIPS2021_958c5305}, we define the forward (and backward) diffusion process that gradually turns a sequence of tokens $\bm{x}$ into special mask tokens (and vice-versa). The diffusion process $\left\{\bm{x}^{t}\right\}$ depends on the time variable $t\in[0,1]$ so that $\bm{x}^{(0)} = \bm{x}$ and $\bm{x}^{(1)}$ is a fully masked sequence. The intermediate values are defined by the probability distribution $q$:
\begin{equation}
    q_{t\mid0}(\bm{x}^{t}\mid\bm{x}) \defeq \prod_{i=1}^{|\bm{x}|}q_{t \mid 0}(x^{t}_i \mid x_i)\text{; where }\,q_{t\mid0}(x^{t}_i\mid x_i) \defeq \begin{cases}
        1 - t, & x^{t}_i=x_i, \\
        t,     & x^{t}_i=\texttt{mask}.
    \end{cases}
\end{equation}

We can see that each token can either remain unchanged or turn into a mask token with probability $t$. The forward process is fully reversible and we can accordingly define the backward process, which gradually unmasks a sequence \citep{NEURIPS2021_958c5305}. Using the results from \cite{ou2025your}, the probability distribution $q_{0|t}(x_i|\,\bm{x}^{t})$ governing the backward process can be modeled with a time-independent transformer language model with parameters $\bm{\theta}$ as $p_{\bm{\theta}}(x_i\,|\,\bm{x}^{t})$. This model can be fitted to the training data by minimizing the upper bound on the negative log-likelihood estimate \citep{ou2025your}:
\begin{equation}
    -\log p_{\bm{\theta}}(\bm{x}) \leq -\int_0^1 \mathop{\mathbb{E}}_{\bm{x}^{t} \sim q_{t\mid0}(\cdot\mid\bm{x})} \Biggl[\frac{1}{t}\sum_{\left\{i\mid x_i^{t}=\,\texttt{mask}\right\}} \!\!\!\!\!\!\!\!\!\log p_{\bm{\theta}}(x_i\mid\bm{x}^{t})\Biggr]\,\mathrm{d}t \defeq {\color{Blue}\mathcal{L}_{\text{MD}}(}\bm{x}{\color{Blue};}\,\bm{\theta}{\color{Blue})}.
\label{eq:diffusion-elbo}
\end{equation}
The integral can be equivalently written as the expectation over $t \sim \mathcal{U}(0, 1)$, thus, it can be directly used as a training objective when estimated by Monte-Carlo sampling \citep{metropolis1949monte}. Such a Monte-Carlo estimate can also be used at inference-time for likelihood-based evaluation, similarly to \Cref{eq:ar}. Note that the resulting objective is very similar to the one used to train masked language models such as BERT \citep{devlin-etal-2019-bert}.

\section{Dual-objective language modeling}\label{sec:objective}

The method of combining autoregressive and masked (diffusion) objectives is mostly based on the earlier GPT-BERT approach by \cite{charpentier-samuel-2024-bert}. They showed promising results for very small language models trained within the limitations of the BabyLM Challenge \citep{conll-2024-babylm}. We extend their approach to masked-diffusion language models and to orders of magnitude larger computation scale.

Dual-objective language models are trained by minimizing the following combined loss function, which is further explained below in more detail:
\begin{equation}
        \mathop{\text{argmin}}_{\bm{\theta}}\mathop{\mathbb{E}}_{\raisebox{-0.3pt}{$_{\bm{x}\, \sim\,\mathcal{D}}$}} \Bigl[{\color{Red}\alpha} {\color{Blue}\mathcal{L}_{\text{AR}}(}\bm{x}{\color{Blue};}\,\bm{\theta}{\color{Blue})} + {\color{Red}(1-\alpha)}{\color{Blue}\mathcal{L}_{\text{MD}}(}\bm{x}{\color{Blue};}\,\bm{\theta}{\color{Blue})}\Bigl].
\end{equation}

\begin{figure}[h!]
    \centering
    \includegraphics[width=\linewidth]{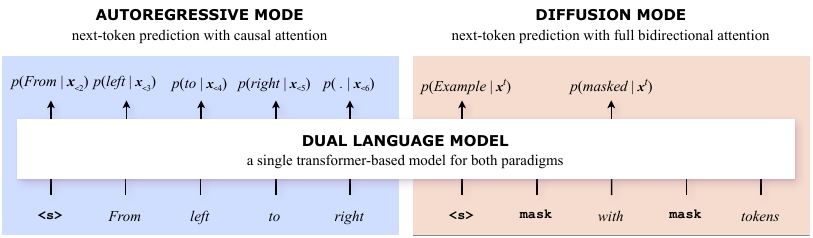}
    \caption{\textbf{Two modes of operation inside a single model.} We use the same transformer architecture with the same parameters to do both diffusion and autoregression language modeling, the only difference between the two modes is the input sequence and the attention mask.}
    \label{fig:dual-lm}
\end{figure}

    \paragraph{Loss weighting} The balance between the autoregressive objective ${\color{Blue}\mathcal{L}_{\text{AR}}}$ and masked-diffusion objective ${\color{Blue}\mathcal{L}_{\text{MD}}}$ is controlled by the hyperparameter ${\color{Red}\alpha}$. It is crucial for controlling the trade-off between training efficiency and overfitting robustness; its relation to the number of data repetitions is extensively tested by the following experiments.
    
    In practice, naively mixing both objectives within a single batch could result in reduced throughput. For this reason, we assign each GPU device to a single objective so that the computation graph remains simple and static, and can be efficiently compiled. To be specific, we distribute the training of each model across $256$ devices, which allows for choosing between $256+1$ values: ${\color{Red}{\alpha}} \in \{\nicefrac{i}{256}\,\vert\,i=0,1,\dots256\}$.

    \paragraph{Diffusion as next-token prediction} Our goal is to align ${\color{Blue}\mathcal{L}_{\text{AR}}}$ and ${\color{Blue}\mathcal{L}_{\text{MD}}}$ so that they can be parameterized by a single transformer model. For this reason, we use a slightly modified version of masked language modeling called \textit{masked next-token prediction} \citep[MNTP;][]{lv-etal-2024-analysis}. With this approach, the model always uses the hidden state at position $i$ to predict the next token at position $i+1$ (we prove that this parameterization is as expressive as the standard approach in \Cref{app:proof}). In this way, both modes of operation are unified as they both perform next-token prediction, as illustrated in \Cref{fig:dual-lm}. MNTP has also been used in recent work for adapting a masked diffusion model from an autoregressive checkpoint \citep{gong2025scaling, ye2025dream7bdiffusionlarge}.
    
    \paragraph{Standard transformer architecture} The main benefits of using masked next-token prediction are that we can use exactly the same transformer architecture as standard autoregressive models, and we can optimize its parameters with both objectives at the same time. The only difference between the two modes of operation is the inputs -- they are either (partially) masked inputs with empty (fully bidirectional) attention masks, or full unchanged inputs with causal (unidirectional) attention masks.

\section{Evaluation}\label{sec:evaluation}

While it is common practice to only consider the value of loss on a held-out set when evaluating language models \citep{kaplan2020scalinglawsneurallanguage, 10.5555/3600270.3602446,NEURIPS2023_9d89448b}, it is important to measure the actual downstream performance to accurately assess the effect of different training configurations. This is especially crucial when training with two incompatible training losses. That being said, we also report validation losses in \Cref{app:validation}.

\paragraph{Tasks} We evaluate our models on nine standard language modeling tasks in a zero-shot fashion. All tasks consist of a context (which can be empty) and multiple different completions where one is correct and the others are incorrect. We evaluate the sum of the log-likelihood of each completion and assign the completion with the maximum sum as the prediction of the model. \Cref{tab:tasks} lists the tasks:

\renewcommand{\arraystretch}{1.2}
\begin{table}[!h]
    \centering
    \caption{\textbf{The list of evaluation tasks.} The ARC$^\dagger$ datasets contain some examples with 3 or 5 completions rather than 4. All tasks are evaluated zero-shot.\vspace{0.5em}}
    \small
    \begin{tabular}{@{}lrrrr@{}}
        \toprule
        \textbf{Task} & \textbf{\# Examples} & \textbf{\# Completions} & \textbf{Split} & \textbf{Reference} \\
        \midrule
         ARC-Easy (ARC-E) & $2\,376$ & $4^\dagger$\hspace{-0.4em} & test & \cite{clark2018thinksolvedquestionanswering} \\
         ARC-Challenge (ARC-C) & $1\,172$ & $4^\dagger$\hspace{-0.4em} & test & \cite{clark2018thinksolvedquestionanswering} \\
         BLiMP & $67\,000$ & $2$ & --- & \cite{warstadt-etal-2020-blimp-benchmark} \\
         Commonsense QA (CSQA) & $1\,221$ & $5$ & val & \cite{talmor-etal-2019-commonsenseqa} \\
         HellaSwag (HSwag) & $10\,042$ & $4$ & val & \cite{zellers-etal-2019-hellaswag} \\
         MMLU & $14\,042$ & $4$ & test & \cite{hendrycks2021measuring} \\
         OpenBook QA (OBQA) & $500$ & $4$ & test & \cite{mihaylov-etal-2018-suit} \\
         Physical Interaction QA (PIQA) & $1\,838$ & $2$ & val & \cite{Bisk_2020} \\
         Social IQa (SIQA) & $1\,954$ & $3$ & val & \cite{sap-etal-2019-social} \\
         \bottomrule
    \end{tabular}
    \label{tab:tasks}
\end{table}

\paragraph{Evaluation setup}

We follow the guidelines of the OLMES paper \citep{gu-etal-2025-olmes} for the normalization of our log-likelihood estimations as well as the prompt format, with two changes: 1) we only evaluate in a zero-shot fashion to simplify the setup, 2) we only consider their ``cloze'' formulation of each task, which is more suitable for smaller models. For the BLiMP task, which is not considered in the OLMES evaluation suite, we do not apply any length normalization and score each sample with the raw log-likelihood score. Since the BLiMP and MMLU tasks contain multiple sub-tasks (67 for BLiMP, and 57 for MMLU), we report their macro-average as the final score. More information on how each task is normalized can be found in \cref{app:norm}.

\paragraph{Normalized score averaging} To ensure a fair aggregation of the different task scores, we first normalize the scores such that the random baseline of each task is at $0$ and the maximum is at $1$; similarly to the Open LLM Leaderboard \citep{open-llm-leaderboard-v2}. To achieve this we apply the following formula to our scores: $\operatorname{score}(x, t) = \nicefrac{(x-r_t)}{(m_t-r_t)}$, where $x$ is the raw score, $r_t$ is the random baseline and $m_t$ is the optimal score for task $t$. We then take the simple average of the normalized scores across all tasks as the final performance of our model.

\subsection{Autoregressive (unidirectional) evaluation}\label{sec:ar_eval}

To evaluate the autoregressive capabilities of our models, we use \cref{eq:ar} to estimate the log-likelihood of each completion. Specifically, given a completion ($\bm{w}$) and context ($\bm{c}$), we calculate the conditional log-likelihood as $\log p_{\bm{\theta}}(\bm{w}\,|\,\bm{c}) = \sum_i\log p_{\bm{\theta}}(w_i \mid \bm{c}, \bm{w}_{<i})$.

\subsection{Masked-diffusion (bidirectional) evaluation}

One possibility to evaluate the masked-diffusion capabilities of our models is to also leverage the training objective in \cref{eq:diffusion-elbo} and estimate the conditional log-likelihood of each completion by Monte-Carlo sampling. We describe this approach in more detail in \Cref{app:monte}. While it provides accurate downstream scores, it is computationally expensive and less accurate than using simpler pseudo log-likelihood \citep[PLL;][]{wang-cho-2019-bert,salazar-etal-2020-masked,10.5555/3737916.3738000} estimation.

PLL allows us to do bidirectional evaluation more than ten times faster while being more accurate than Monte-Carlo sampling (\cref{app:mntp_diff}). Therefore, we use PLL for evaluating the bidirectional capability of our models. We fully describe this method in \Cref{app:pll}. As visualized in \Cref{fig:masked_log} on the left, we specifically use the semi-autoregressive variation of PLL proposed by \cite{10.5555/3737916.3738000}.

\begin{figure}[h!]
    \centering
    \includegraphics[width=\linewidth,trim={0 0 0 0},clip]{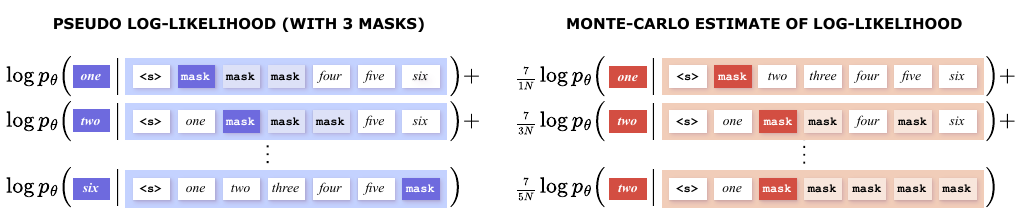}
    \caption{\textbf{Visual representations of bidirectional evaluation methods.} Pseudo log-likelihood estimation (on the left) reaches accurate likelihood scores substantially faster than the (theoretically grounded) Monte-Carlo estimation (on the right).}
    \label{fig:masked_log}
\end{figure}

\section{Experiments}

\subsection{Pretraining setup}\label{sec:pretrain}

We train each 470-million-parameter language model (with 360M non-embedding weights) on 32 billion tokens in total. A repetition factor $R$ means we sample a unique subset of size $\nicefrac{\text{32B}}{R}$ tokens and repeat it $R$ times during training. This total token budget is more than $4\times$ past the Chinchilla compute-optimal point \citep{10.5555/3600270.3602446}; we specifically decided to conduct the experiments in this regime as it reflects how modern language models are trained in practice. This compute budget is also large enough to induce non-trivial zero-shot downstream performance, enabling us to measure clear differences between different configurations.

\paragraph{Model architecture} The language models have 24 layers with hidden size of 1\,024, their self-attention operations are divided into 16 parallel heads, the feed-forward modules have intermediate size of 3\,554, and the vocabulary is set to 51\,200 tokens. As for the architecture itself, we follow the usual modifications of the original transformer recipe \citep{NIPS2017_3f5ee243} -- pre-normalization \citep{nguyen-salazar-2019-transformers} with RMSNorm \citep{10.5555/3454287.3455397}, rotational positional embedding \citep{10.1016/j.neucom.2023.127063} and Swish-gated linear units \citep{ramachandran2018searching, shazeer2020gluvariantsimprovetransformer}.

\paragraph{Optimization} The parameters are optimized by the Muon optimizer for faster convergence \citep{jordan2024muon}, specifically its variation proposed by \cite{liu2025muonscalablellmtraining}. The learning rate is set to 0.007 and decayed according to the warmup-stable-decay \citep[WSD;][]{hagele2024scaling} schedule (without warmup steps and 2\,048 steps of linear decay). In total, each model is trained for 8\,192 steps with 4M tokens in each global batch and with a sequence length of 2\,048 tokens. The optimization is regularized by weight decay (with strength of $10^{-1}$) and by an auxiliary z-loss term \citep[with strength of $10^{-4}$;][]{chowdhery2022palmscalinglanguagemodeling}.

\paragraph{Training corpus and tokenizer} Even though we limit the training data to 32B tokens, we deliberately choose a text corpus that is not excessively filtered and that is representative of large-scale web crawls used in practice. We randomly sample English documents with 32B tokens in total from the HPLT v2 corpus \citep{burchell-etal-2025-expanded}, which combines extracted webpages from the Internet Archive and CommonCrawl. We also use a smaller disjoint subset to monitor the validation loss. To prevent a potential bias from using an external tokenizer, we train a standard byte-level BPE tokenizer \citep{10.5555/177910.177914} with 51\,200 subwords directly on the full training data.

\begin{figure}[h!]
    \centering
    \includegraphics[width=\linewidth]{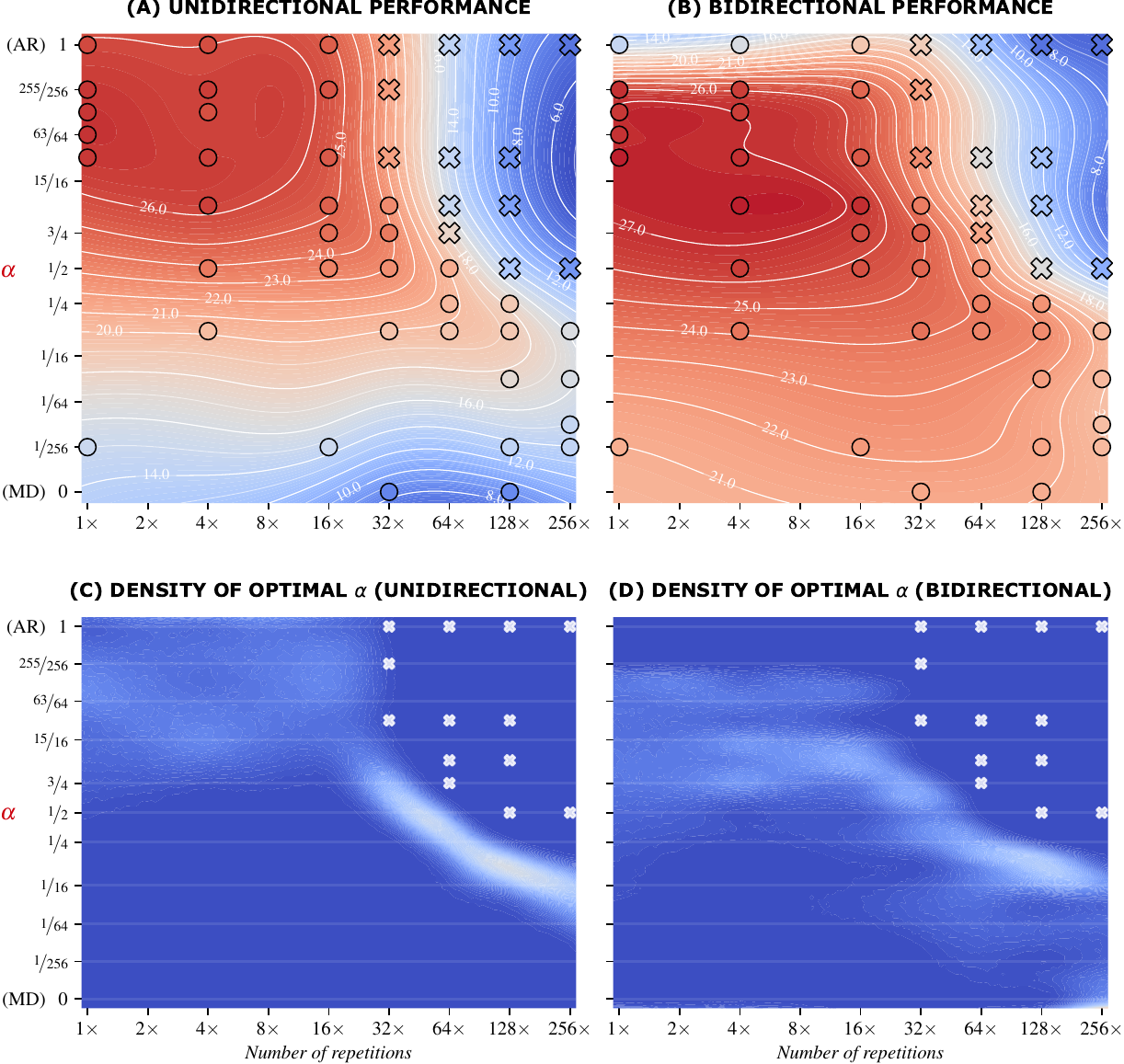}
    \caption{\textbf{Interpolated unidirectional and bidirectional results.} The (a) and (b) figures on top show the relation between repetitions (x-axis) and the autoregressive-diffusion weight ${\color{Red}\alpha}$ (y-axis); the contours follow the Gaussian process model that interpolates the average performance of language models trained according to the specified settings. The respective results are plotted either as crosses when the model overfitted during training, or as circles. The (c) and (d) figures below visualize the estimated probability that a particular ${\color{Red}\alpha}$ (y-axis) is optimal for a given number of repetitions (x-axis).}
    \label{fig:contour}
\end{figure}

\subsection{Searching for the optimal ${\color{Red}\alpha}$}

We trained and evaluated 50 language models in total, the results are plotted in \Cref{fig:contour}. In order to deal with the noisy nature of this data and to better understand the relation between the amount of data repetitions and the optimal ${\color{Red}\alpha}$, we use simple statistical models.

\paragraph{Interpolation with Gaussian process} Specifically, we use Gaussian process regression \citep[GPR;][]{NIPS1995_7cce53cf} with a composite kernel structure to model the relationship between data repetitions, ${\color{Red}\alpha}$ and downstream performance. The composite kernel consists of a constant kernel multiplied by an anisotropic Matérn kernel \citep[$\nu=1.5$;][]{stein1999interpolation} combined additively with a white noise kernel to account for observation noise. The input features are standardized to zero mean and unit variance, and the output features are normalized. The kernel parameters are optimized by L-BFGS-B \citep{10.5555/3112655.3112866} using \texttt{SciPy} \citep{2020SciPy-NMeth}. The resulting interpolations in \Cref{fig:contour} show regular structure while closely fitting the data with $R^2$ over 0.99 in all cases.

\paragraph{The optimal autoregressive-diffusion ratios} The fitted Gaussian process is a probabilistic model of the downstream performance with regard to the amount of data repetition and ${\color{Red}\alpha}$. Thus, we can transform this to the probability that a particular ${\color{Red}\alpha}$ is optimal for the given data repetition. More concretely, we can estimate the density of this distribution by sampling from the posterior of the GPR model. The result of this is visualized in the bottom part of \Cref{fig:contour}.

\subsection{Results and discussion} 

The structure of \Cref{fig:contour} becomes clearer once we identify which training settings result in overfitting during training.\footnote{Here, \textit{overfitted training runs} are those runs, in which the held-out loss starts diverging while the training loss keeps converging (\Cref{app:validation}). Such runs are highlighted in \Cref{fig:contour} by $\bm{\times}$ marks.} The density of optimal ${\color{Red}\alpha}$ weights highlights that there are two regions to consider: \textit{1) Regular-data region} where a language model trained solely on the autoregressive objective does not overfit -- this roughly corresponds to 16 repetitions of training data and less, as also shown by \cite{NEURIPS2023_9d89448b}. \textit{2) Data-constrained region} -- roughly corresponding to 32 data repetitions and more -- where overfitting is an important consideration.

In the first case, it is clearly beneficial to put more weight on the autoregressive training than on masked-diffusion. Yet, training only autoregressively does not lead to any improvement in any experiments within the regular-data region. Even when evaluated purely autoregressively, the differences between ${\color{Red}\alpha}$ set to $1$ and $\nicefrac{15}{16}$ are negligible. Switching to bidirectional evaluation, the single-objective ${\color{Red}\alpha}=1$ performs poorly while all models trained with ${\color{Red}\alpha}$ values between $\nicefrac{255}{256}$ and $\nicefrac{15}{16}$ perform similarly -- notably, they all substantially outperform models trained only with masked-diffusion. This is a key finding: even without any data constraint, the dual-objective models achieve stronger masked-diffusion performance than pure masked-diffusion training, despite dedicating only a small fraction of training to the masked-diffusion objective. We hypothesize that the prevalence of the autoregressive objective leads to fast convergence, and that the small amount of masked-diffusion balances its slower convergence by inducing useful modeling priors. This leads us to formulating the first practical recommendation:

\vspace{0.5em}
\begin{remark}{\hspace{1em}Language modeling under regular data settings}{rem:regular}
When training a language model in a regular data setting (16 repetitions or less), train with a small amount of masked-diffusion objective (${\color{Red}\alpha}\approx\nicefrac{63}{64}$) to achieve stronger bidirectional performance than pure masked-diffusion training without losing any autoregressive performance.
\label{rem:regular}
\end{remark}
\vspace{0.5em}

In the second data-constrained case, the relation between data repetition, ${\color{Red}\alpha}$, and final performance seems more complicated. We risk overfitting by putting too much weight on autoregression and underfitting by focusing too much on masked-diffusion; as evident from \Cref{fig:contour}, the interval of optimal ${\color{Red}\alpha}$ values is fairly narrow. On the other hand, the optimal values are surprisingly similar for the unidirectional and bidirectional performance. We can notice that the region of optimal ${\color{Red}\alpha}$ values is right beneath the region of ${\color{Red}\alpha}$ values that lead to overfitting, but the question is how to identify such an ${\color{Red}\alpha}$. It is possible to have an alternative interpretation of the autoregressive-diffusion weights and count the number of data repetitions that each objective is individually trained on -- then we can see that more than 32 autoregressive repetitions lead to overfitting while fewer than 8 autoregressive repetitions lead to underfitting. Thus, based on the empirical results, our recommendation for this scenario is:

\vspace{0.5em}
\begin{remark}{\hspace{1em}Data-constrained language modeling}{rem:constrain}
When training a language model in a data-constrained setting (more than 32 repetitions), choose~${\color{Red}\alpha}$ that exposes the autoregressive objective to roughly 16 repetitions of the training data.
\label{rem:constrain}
\end{remark}
\vspace{0.5em}

\renewcommand{\arraystretch}{1.0}
\begin{table}[h]
    \centering
    \small
    \caption{\textbf{The normalized autoregressive performance of selected models.} We show the results on all nine evaluated tasks for three repetition values; each repetition group contains the results of the best-performing ${\color{Red}\alpha}$ and of the autoregressive-only model. The scores for each task are normalized so that 0\% corresponds to random baseline and 100\% is the perfect score. The best result for each dataset size is boldfaced.\vspace{1em}}
    \begin{tabular}{@{}lccccccccc@{\hspace{2.25em}}c@{}}
        \toprule
        \raisebox{0em}{\textbf{Model configuration}} & \rotatebox{90}{\textbf{ARC-C}} & \rotatebox{90}{\textbf{ARC-E}} & \rotatebox{90}{\textbf{BLiMP}} & \rotatebox{90}{\textbf{CSQA}} & \rotatebox{90}{\textbf{HSwag}} & \rotatebox{90}{\textbf{MMLU}} & \rotatebox{90}{\textbf{OBQA}} & \rotatebox{90}{\textbf{PIQA}} & \rotatebox{90}{\textbf{SIQA}} & \rotatebox{0}{\textbf{Average}} \\
        \midrule

         \scriptsize{\textsc{1 repetition}} \\
         \hspace{1em}Dual (${\color{Red}\alpha}=\nicefrac{63}{64}$)  & \hphantom{0}5.7 & 28.6 & \textbf{63.7} & \textbf{35.1} & 31.1 & \textbf{\hphantom{0}4.9} & \textbf{17.6} & \textbf{40.9} & 14.3 & \textbf{26.9} \\
         \hspace{1em}Autoregressive (${\color{Red}\alpha}=1$) & \textbf{\hphantom{0}5.9} & \textbf{30.3} & 61.3 & 33.5 & \textbf{31.7} & \hphantom{0}3.8 & 13.6 & 39.4 & \textbf{15.2} & 26.1 \\[0.5em]

         \scriptsize{\textsc{32 repetitions}} \\
         \hspace{1em}Dual (${\color{Red}\alpha}=\nicefrac{3}{4}$) & \hphantom{0}3.3 & \textbf{28.0} & \textbf{57.9} & \textbf{31.1} & \textbf{26.4} & \hphantom{0}3.6 & \textbf{14.4} & \textbf{36.1} & \textbf{14.6} & \textbf{23.9} \\
         \hspace{1em}Autoregressive (${\color{Red}\alpha}=1$) & \textbf{\hphantom{0}5.0} & 24.9 & 53.3 & 28.5 & 25.4 & \textbf{\hphantom{0}3.8} & \hphantom{0}9.9 & 33.3 & 14.2 & 22.0 \\[0.5em]

        \scriptsize{\textsc{128 repetitions}} \\
         \hspace{1em}Dual (${\color{Red}\alpha}=\nicefrac{1}{8}$) & \textbf{\hphantom{0}1.7} & \textbf{23.6} & \textbf{56.1} & \textbf{24.8} & \textbf{14.2} & \textbf{\hphantom{0}1.6} & \textbf{\hphantom{0}8.5} & \textbf{28.1} & \textbf{13.3} & \textbf{19.1} \\
         \hspace{1em}Autoregressive (${\color{Red}\alpha}=1$) & \,-1.0 & 12.3 & 33.2 & \hphantom{0}6.8 & \hphantom{0}8.1 & \hphantom{0}1.1 & \,-0.5 & 15.8 & \hphantom{0}8.9 & \hphantom{0}9.4 \\
         
         \bottomrule
    \end{tabular}
    \label{tab:performance}
\end{table}

\paragraph{Generalization to larger language models} An obvious question is whether the recommendations hold even at much bigger scale for larger language models. Reliably answering this question would require expensive experimentation, but we believe that the conclusions hold for two reasons. Firstly, according to our results, the optimal ${\color{Red}\alpha}$ values are clearly correlated with overfitting of autoregressive language models. Since the overfitting behavior does not depend on model size according to previous work \citep{NEURIPS2023_9d89448b,prabhudesai2025diffusionbeatsautoregressivedataconstrained}, we believe that the optimal ${\color{Red}\alpha}$ values should also not change. Secondly, the relative burden of representing two modes of operation within the learned parameters decreases with model size, so we believe that the benefit of the dual training objective should actually increase with model size.

\paragraph{Detailed results} To put the abstract average scores into another perspective, we look at the individual (normalized) scores per task in \Cref{tab:performance}. The results show that the improvement in performance from using a dual objective is observed on a majority of tasks. This is especially true the more repetitions there are. The detailed scores also highlight how effectively the dual objective learns from limited data, reaching nontrivial performance even when exposed to just 256M tokens of training data (under 128 repetitions). We observe similar trends for masked-diffusion evaluation except that as the number of repetitions decreases, the performance gap increases rather than decreases. Detailed performance for the masked-diffusion evaluation can be found in \cref{app:mntp_perf}.

\subsection{Generalization to prefix language modeling}

Prefix language modeling \citep{NEURIPS2019_c20bb2d9, 10.5555/3455716.3455856, pmlr-v162-wang22u} is a promising alternative to the two training objectives investigated in this work. It processes the conditioning part (prefix, $\bm{c}$ in notation from \Cref{sec:ar_eval}) of a text fully bidirectionally while the completion part ($\bm{w}$ in \Cref{sec:ar_eval}) is processed autoregressively. Given that our models are trained with both unidirectional and bidirectional attention, we test whether the exposure to both can induce generalization to prefix language modeling without any further training. We repeat the earlier autoregressive evaluation with prefix attention masks and plot the results in \cref{fig:prefix}.

\begin{figure}[h!]
    \centering
    \includegraphics[width=\linewidth]{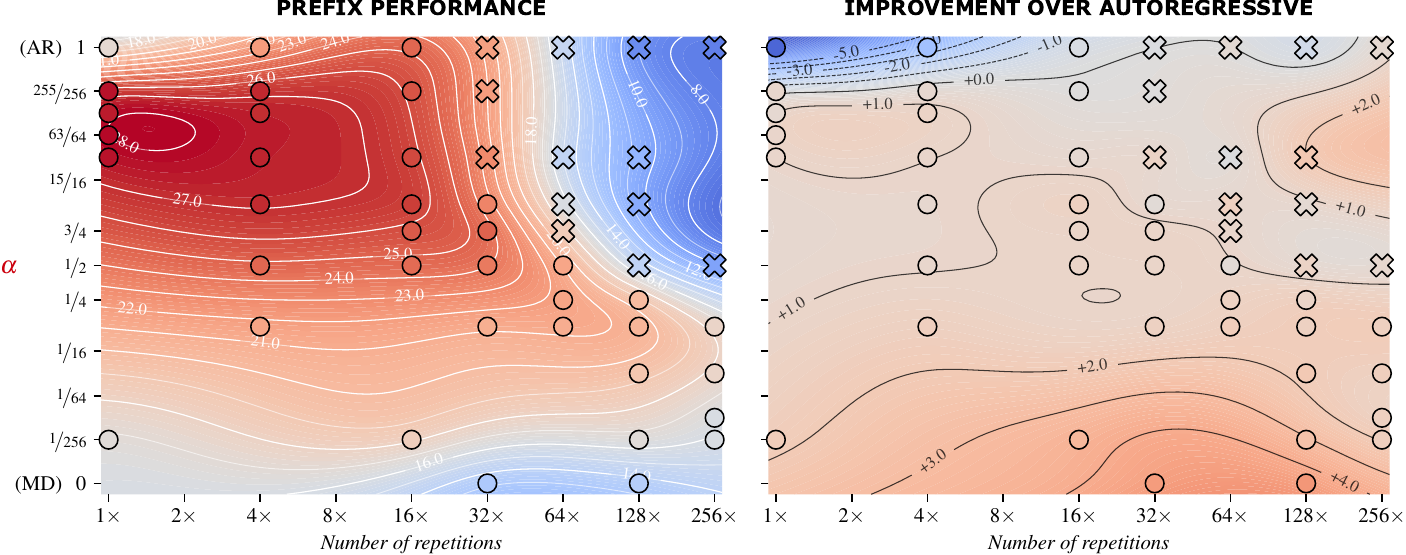}
    \caption{\textbf{Interpolated prefix results.} The figures show the relation between data repetitions (x-axis), ${\color{Red}\alpha}$ (y-axis), and downstream performance (color-coded). The individual results are interpolated by a GPR model. The right figure demonstrates the relative improvement of prefix-masked evaluation compared to fully unidirectional evaluation (blue color denotes decreased performance and red color denotes a performance increase).}
    \label{fig:prefix}
\end{figure}

The right side of \cref{fig:prefix} shows the overall improvement of the prefix evaluation over the autoregressive one. Notably, we can see that it is reliably over one percentage point better across most configurations that combine both training objectives. This finding leads to our third recommendation:

\vspace{0.5em}
\begin{remark}{\hspace{1em}Induced prefix language modeling}{remark:3}
    The autoregressive performance of dual-objective language models can be reliably improved at inference time by processing the conditional part of a prompt fully bidirectionally.
\label{remark:3}
\end{remark}
\vspace{0.5em}

\section{Related work}

\paragraph{Combining autoregressive and masked (diffusion) language modeling} This paper builds upon the GPT-BERT training objective by \cite{charpentier-samuel-2024-bert}, validating its effectiveness in a more practical setting. However, there is a long history of papers that tried to combine bidirectional masked language modeling with unidirectional autoregressive modeling: T5 \citep{10.5555/3455716.3455856} and BART \citep{lewis-etal-2020-bart} were the first to train with autoregressive fill-in-the-blank training objectives by relying on encoder-decoder transformer architectures. Later, \cite{du-etal-2022-glm} proposed GLM, which uses the same objective as T5 while using a simpler decoder-only architecture with a complicated scheme of positional encodings. CM3 by \cite{aghajanyan2022cm3} further simplifies training by not requiring any non-standard architectural modifications like the previous work. As they also add autoregressive language-modeling objective, their work is close to our approach -- a model trained with CM3 can be used as any other autoregressive model at inference time, similarly to us. However, our objective also generalizes masked-diffusion language modeling and allows for fine-grained balance of the two objectives throughout training. More recently AntLM by \citet{yu-etal-2024-antlm} proposed to switch from one objective to the other in a curriculum fashion, starting with a short autoregressive training, followed by a long masked language training and finishing on another short autoregressive training. While this does show promise, the transition from one objective to the other leads to forgetting of the previous objective whereas our objective continuously learns both objectives. Other notable works include prefix language models \citep{NEURIPS2019_c20bb2d9, 10.5555/3455716.3455856, pmlr-v162-wang22u} and UL2 \citep{tay2023ul}.

\paragraph{Scaling of autoregressive and masked-diffusion models} Concurrent works by \cite{prabhudesai2025diffusionbeatsautoregressivedataconstrained} and \cite{ni2025difflm} have demonstrated that masked-diffusion models outperform autoregressive models in data-constrained training regimes. Our results confirm their findings but we show that using either of these training objectives is never optimal -- combining them together should always be better, not only in data-constrained settings.

\paragraph{Bidirectional masking of user and system prompts} A recent paper by \citet{katz-etal-2025-segment} shows that using a bidirectional mask on user and system prompts improves performance on a wide variety of tasks, in line with \Cref{remark:3}. However, for models to be able to use such masks, the authors first need to train adapters. Our work shows that by training both autoregressive and masked-diffusion at the same time, we are able to induce the prefix mask without any additional training.

\paragraph{Data-constrained scaling laws} \cite{NEURIPS2023_9d89448b} studies the scaling laws of autoregressive models in data-constrained settings with a similar motivation to this paper. They show that autoregressive models cannot meaningfully learn from more than 16 data repetitions -- we demonstrate that this value is at least an order of magnitude larger when training with the dual objective.

\paragraph{Autoregressive diffusion} Our work shares motivation with the autoregressive-diffusion models proposed by \cite{10.5555/3666122.3667859}. The diffusion process in that work is biased towards left-to-right denoising, which improved the decoding efficiency of the diffusion language models at that time. Similarly, \cite{arriola2025block} speeds up decoding of masked-diffusion models by autoregressively generating chunks of tokens where each chunk is decoded by a diffusion process. In both cases, the resulting models are still diffusion models -- albeit faster; these approaches do not generalize over autoregressive and masked-diffusion language modeling as our method.

\paragraph{Fair MD-AR comparison} The recent work by \cite{xue2025anyorder} modifies masked-diffusion language models by parameterizing them with causally-masked transformers, which makes the diffusion models more comparable to standard autoregressive models -- decoupling their architectural differences from differences in training objectives. Their conclusion is that masked diffusion alone is a suboptimal objective for language, which is also confirmed by our experiments (\Cref{fig:contour}). However, we found that by simply combining both objectives, we can get the benefits of diffusion without losing any performance.

\paragraph{Approaching the data wall} Large language models are known to reliably follow the empirical \textit{scaling laws} that describe how their performance should improve with increased compute, model size, and training data. \cite{kaplan2020scalinglawsneurallanguage} first demonstrated these relationships, showing how the training loss decreases as a power law with respect to these three parameters. These laws were later refined by \citet{10.5555/3600270.3602446}, who showed that compute-optimal training requires scaling data and model size together. Related to our work, the scaling laws reveal a fundamental problem: achieving each incremental gain in performance requires exponentially more training data. Thus, data-constrained language modeling is quickly becoming a relevant field of study even for high-resource languages such as English.

\section{Conclusion}

In this work, we addressed the fundamental trade-off between the training efficiency of autoregressive models and the overfitting resilience of masked-diffusion models. We have empirically demonstrated that a dual-objective training strategy successfully achieves the best of both worlds, resulting in models that converge rapidly without any performance degradation in data-constrained settings. Crucially, because this unification requires no architectural changes, the resulting models incur no inference overhead and can be deployed as standard autoregressive transformers.

We established that combining objectives is universally beneficial and derived practical guidelines for selecting the optimal ${\color{Red}\alpha}$ based on the degree of data repetition. Furthermore, we observed that the diffusion objective induces robust prefix language modeling capabilities, leading to superior performance on downstream tasks compared to standard autoregressive baselines. While training on hundreds of data repetitions may seem extreme today, the asymmetry between exponentially scaling compute budgets and the finite supply of high-quality text suggests that data constraints will become increasingly relevant for frontier model development. Our findings indicate that dual-objective training provides a robust and compute-efficient path forward that retains standard inference capabilities as the field approaches these fundamental limits.

\vspace{1em}

\subsubsection*{Reproducibility Statement}
To ensure reproducibility of our work we provided the guidelines on how to train language models on both objectives at the same time in \cref{sec:objective}. For our model parameters and hyperparameters we specified those in \cref{sec:pretrain}. We describe how we perform the evaluations, the number of mask tokens used for PLL, the prompt formats, and log-likelihood normalizations in \cref{sec:evaluation}, \cref{app:norm}, and \cref{app:pll}. We openly release our custom training and evaluation code at {\footnotesize\url{https://github.com/ltgoslo/dual-language-models}}. The training code is based on the common and freely distributed deep-learning framework \texttt{PyTorch} \citep{10.5555/3454287.3455008}. The trained models are released openly under the Apache 2.0 license for further investigation at {\footnotesize\url{https://huggingface.co/ltg/dual-lm-470m}}.

\subsubsection*{Author Contributions}
Both authors have contributed equally and should be considered shared first authors of this manuscript.

\subsubsection*{Acknowledgments}

The computations were performed on resources provided through Sigma2 – the national research infrastructure provider for high-performance computing and large-scale data storage in Norway. We acknowledge Norway and Sigma2 for awarding this project access to the LUMI supercomputer, owned by the EuroHPC Joint Undertaking, hosted by CSC (Finland) and the LUMI consortium through project 465001890.

The efforts described in this paper were jointly funded by the University of Oslo and the HPLT project (High Performance Language Technologies; coordinated by Charles University).


\bibliography{iclr2026_conference}
\bibliographystyle{iclr2026_conference}

\newpage

\appendix
\section{The use of large language models}

Large language models have been used to provide feedback, fix grammatical errors and improve the writing in this paper; in particular, we used the Claude family of language models from {\footnotesize\url{https://claude.ai}}. In addition, we used the autocompletion tool from GitHub Copilot when writing the code used in this work.

\section{Erratum: Loss formulation}

The original version of this paper as well as the training code used the following formulation the loss function:
\begin{equation}
        \mathop{\text{argmin}}_{\bm{\theta}}\mathop{\mathbb{E}}_{\raisebox{-0.3pt}{$_{\bm{x}\, \sim\,\mathcal{D}}$}} \Bigl[{2\color{Red}\alpha} {\color{Blue}\mathcal{L}_{\text{AR}}(}\bm{x}{\color{Blue};}\,\bm{\theta}{\color{Blue})} + {\color{Red}(1-\alpha)}{\color{Blue}\mathcal{L}_{\text{MD}}(}\bm{x}{\color{Blue};}\,\bm{\theta}{\color{Blue})}\Bigl].
\end{equation}

The additional factor of two for ${\color{Blue}\mathcal{L}_{\text{AR}}}$ was supposed to balance the $\nicefrac{1}{t}$ term applied when computing ${\color{Blue}\mathcal{L}_{\text{MD}}}$ (recall \Cref{eq:diffusion-elbo}). This is however not mathematically correct (\Cref{eq:diffusion-elbo}) and unnecessary; the loss function is thus simplified in this updated version.

\section{Log-likelihood normalization}\label{app:norm}

For the BLiMP task, which is not considered in the OLMES evaluation suite, we do not apply any normalization and take the raw log-likelihood. We also stick to the no-context form of this task, where the whole sentence is considered the completion. We apply character length normalization to ARC-Easy, HellaSwag, MMLU, PIQA, and SIQA. Finally, we apply point-wise mutual information normalization \citep{holtzman-etal-2021-surface}, where the log-likelihood of the context-informed completion is divided by the log-likelihood of the uncontrained context completion, this can be seen in \cref{eq:pmi}, to ARC-Challenge, Commonsense QA, and OpenBook QA.
\begin{equation}
    \operatorname{PMI}(\bm{w}) = \sum_{i=1}^{|\bm{w}|}\log\left(\frac{p_{\bm{\theta}}\left(w_i\mid\bm{c} \oplus \bm{w}_{<i}\right)}{p_{\bm{\theta}}\left(w_i\mid\bm{u} \oplus \bm{w}_{<i}\right)}\right),
    \label{eq:pmi}
\end{equation}
where $\bm{w}$ is the completion, $\bm{c}$ is the context, and $\bm{u}$ is the unconstrained context (in our case, this would be ``Answer:'')

\section{Monte Carlo estimation of log-likelihood}
\label{app:monte}

To evaluate the masked-diffusion capabilities of our models, we use \cref{eq:diffusion-elbo} with the same modification as for the autoregressive evaluation as well as an adaptation of Monte-Carlo sampling to estimate the log-likelihood of each completion. Specifically, we estimate the following expected value:
\begin{equation}
    \int_0^1 \mathop{\mathbb{E}}_{\bm{x}^{t} \sim q_{t\mid0}(\cdot\mid\bm{x})} \Biggl[\frac{1}{t}\sum_{\left\{i\mid x_i^{t}=\,\texttt{mask}\right\}} \!\!\!\!\!\!\!\!\!\log p_{\bm{\theta}}(x_i\mid\bm{x}^{t})\Biggr]\,\mathrm{d}t = \mathop{\mathbb{E}}_{\substack{t \sim \mathcal{U}(0,1) \\ \bm{x}^{t} \sim q_{t\mid0}(\cdot\mid\bm{x})}} \Biggl[\frac{1}{t}\sum_{\left\{i\mid x_i^{t}=\,\texttt{mask}\right\}} \!\!\!\!\!\!\!\!\!\log p_{\bm{\theta}}(x_i\mid\bm{x}^{t})\Biggr]
\end{equation}

To reduce the variance of the estimation and get faster convergence, we take the expectation between $N$ equally spaced points between 0 and 1. instead of taking the expectation over $t \sim \mathcal{U}(0, 1)$. Yet, accurate estimation still requires $N\geq256$, which is unbearably slow -- especially when compared to simple autoregressive calculation of log-likelihood that requires only a single forward pass.

\section{Pseudo log-likelihood estimation}
\label{app:pll}

The base PLL equation can be described by a slight modification of \cref{eq:ar}:
\begin{equation}
    \begin{aligned}
    \log p_{\bm{\theta}}(\bm{w}) \approx \sum_{i=1}^{|{w}|}\log p_{\bm{\theta}}\bigl(w_i\mid \bm{c} & \oplus w_0 \oplus \cdots \oplus w_{i-1} \\[-1em]
    & \oplus \mathord{\texttt{mask}}\\
    & \oplus w_{i+1} \oplus \cdots \oplus w_{|\bm{w}|}\bigr)
    \end{aligned}
    \label{eq:pll_eval}
\end{equation}
This means that instead of doing a single forward pass, we need to do $|\bm{w}|$ forward passes to estimate the PLL. However, using a single mask token could lead to underestimating the log-likelihood of words split into multiple tokens. Therefore, we can further modify \cref{eq:pll_eval} to have a variable (but constant) number of mask tokens after the token we are trying to estimate:
\begin{equation}
    \begin{aligned}
    \log p_{\bm{\theta}}(\bm{w}) \approx \sum_{i=1}^{|\bm{w}|}\log p_{\bm{\theta}}\bigl(w_i \mid \bm{c} & \oplus w_0 \oplus \cdots \oplus w_{i-1} \\[-1em]
    & \oplus \texttt{mask} \oplus \cdots \oplus \texttt{mask} \\
    & \oplus\, w_{i+n} \oplus \cdots \oplus w_{|\bm{w}|}\bigr),
    \end{aligned}
\end{equation}
where $n$ represents the number of $\operatorname{[MASK]}$ tokens.
In our case, we take a combination of two different numbers of mask tokens (1 and 6), by taking the best score of the two for each task. The two values were chosen experimentally, more details on the results of each number of mask tokens can be found in \cref{app:num_tokens}.

\section{Proof of left-shift closure}
\label{app:proof}

This section proves that when we parameterize masked-diffusion language models as bidirectional transformers with shifted output, we do not lose any expressivity compared to standard non-shifted bidirectional models. We prove it constructively by defining a shift operation in the RASP language (which can then be compiled into an equivalent transformer model).

\vspace{0.5em}

\begin{definition}[RASP programs] The Restricted Access Sequence Processing language \citep[RASP;][]{pmlr-v139-weiss21a} is a sequence processing language that uses two types of variables: \textit{sequence operators} and \textit{selectors}; and two types of
operators: \textit{element-wise} and \textit{select-aggregate} operators. Valid \textit{programs} in RASP are operations on sequence operators formed by a finite composition of element-wise and select-aggregate operators.
\vspace{-0.5em}
\begin{itemize}[align=parleft,left=1em..2em]
\item \textit{Sequence operators} represent sequences of values (akin to hidden states in transformer models). \texttt{tokens} and \texttt{indices} are two pre-defined sequence operators; the first directly returns a sequence of the input tokens ($\texttt{tokens}(\texttt{"hello"}) = [\texttt{h}, \texttt{e}, \texttt{l}, \texttt{l}, \texttt{o}]$, and the second returns the positional indices ($\texttt{indices}(\texttt{"hello"}) = [0, 1, 2, 3, 4]$).
\item \textit{Selectors} are binary matrices (akin to attention matrices in transformers).
\item \textit{Element-wise operators} are arbitrary element-wise transformations on sequence operators (akin to feed-forward layers in transformers). For example $(\texttt{indices} + 2)(\texttt{"hello"}) = [2, 3, 4, 5, 6]$.
\item \textit{Select-aggregate operators} consist of two sequentially applied operators \texttt{select} and \texttt{aggregate} (corresponding to the attention operation).
\item $\texttt{select}(\bm{x}, \bm{y}, p)$ is an operator defined on two sequence operators $\bm{x}$ and $\bm{y}$, and an element-wise boolean operator $p$ defined on two sequence operators; the result is a selector matrix $\bm{M}$, where $M_{ij}=p(x_i, y_j)$. For example, $\texttt{select}([0, 1, 2], [1, 2, 3], <)$ results in a upper-triangular $3\times3$ binary matrix (selector).
\item $\texttt{aggregate}(\bm{M}, \bm{x}; c)$ is an operator defined on a selector $\bm{M}$, a sequence operator $\bm{x}$ and a default value $c$ (usually set to $0$ and omitted for convenience). It produces a sequence operator $\bm{y}$ such that:
\[
y_i=
\begin{cases}
    \frac{1}{\left|\left\{j:\,M_{ij}=1\right\}\right|}\sum_{j:\,M_{ij}=1}x_j, & \text{if }\,|\{j:\,M_{ij}=1\}| > 0,\\[0.5em]
    c, & \text{otherwise.}
\end{cases}
\]
\end{itemize}
\end{definition}
\vspace{0.75em}

\begin{fact}[RASP-transformer reduction]
\label{fact:rasp-transformer}
For every valid program written in RASP, there exists an equivalent fully-bidirectional transformer model that computes the same per-position operation; see \cite{pmlr-v139-weiss21a, 10.5555/3666122.3667771}.
\end{fact}
\vspace{0.75em}

\begin{definition}[$\Sigma$-realizable functions]
\label{def:realizable}
We consider programs defined on an input alphabet $\Sigma$ with a special token $\texttt{<s>}\in\Sigma$. A valid input sequence $\bm{x}=(x_1,x_2\dots x_n)\in\mathcal{X}$ is every sequence where $x_1=\texttt{<s>}$ and all $x_i\in\Sigma$. The output space $\mathcal{Y}$ is made of sequences $\bm{y}=(y_1,y_2\dots y_n)\in\mathcal{Y}$, where every element is a probability distribution over the alphabet $\Sigma$: that is all $y_i\in[0,1]^{|\Sigma|}$ and $\sum_j{(y_i)_{j}=1}$.

A function $f:\mathcal{X}\to\mathcal{Y}$ is $\Sigma$-\textit{realizable} if there exists a transformer whose output on every input $\bm{x}\in\mathcal{X}$ equals $f(\bm{x})$ position-wise. Let $\mathcal{R}_{\Sigma}$ be the class of all $\Sigma$-realizable functions.
\end{definition}
\vspace{0.75em}

\begin{theorem}[Left-shift closure]
\label{theorem:closure}
$\mathcal{R}_{\Sigma}$ is closed under unit left-shifts: for every $f\in\mathcal{R}_{\Sigma}$, there exists $g\in\mathcal{R}_{\Sigma}$ such that for all $\bm{x}\in\mathcal{X}\text{ and }i\in\left[1,n-1\right]\!:\,g(\bm{x})_i=f(\bm{x})_{i+1}$ (note that $f(\bm{x})_{1}$ and $g(\bm{x})_{n}$ are not constrained).
\end{theorem}
\begin{proof}
The proof constructs a suitable function $g\in\mathcal{R}_{\Sigma}$ for any $f\in\mathcal{R}_{\Sigma}$. The new function $g$ will mirror function $f$ and then shift its output so that $g(\bm{x})_i=f(\bm{x})_{i+1}$, the shift will be constructed in RASP so that $g$ is $\Sigma$-realizable.

Let $f\in\mathcal{R}_{\Sigma}$ be any $\Sigma$-realizable function and set $T_{\!f}$ as a fully-bidirectional transformer that realizes $f$, so $T_{\!f}(\bm{x})_i=f(\bm{x})_i$ for all valid inputs $\bm{x}\in\mathcal{X}$ and all positions $i\in[1,n]$.

First, we define a RASP selector $\bm{S}=\texttt{select}(\texttt{indices}+1,\ \texttt{indices},\ =)$, whose entries therefore satisfy $S_{ij}=1$ iff $j=i+1$ (each row $i$ selects exactly the next position $i+1$, and the last row selects none).

Then, for any sequence operator $\bm{z}$ (possibly vector-valued), we define a RASP program
$ \texttt{shift}(\bm{z})=\texttt{aggregate}(\bm{S},\,\bm{z};\, c)$, where $c$ is arbitrary and can be simply set to $z_n$. By construction of $\bm{S}$ and the definition of $\texttt{aggregate}$, we have $\texttt{shift}(\bm{z})_n=c=z_n$ and for every $i\in[1,n-1]$:
\begin{equation}
\texttt{shift}(\bm{z})_i
=
\frac{1}{\lvert\{j:\,S_{ij}=1\}\rvert}\sum_{j:\,S_{ij}=1}z_j
= z_{i+1}.
\end{equation}

Using \hyperref[fact:rasp-transformer]{Fact~\ref{fact:rasp-transformer}}, there exists a transformer $T_{\!\texttt{shift}}$ that computes the RASP program \texttt{shift}. Therefore, we can construct a transformer $T_{\!g}$ as $T_{\texttt{shift}} \circ T_{\!f}$. This corresponds to the function $g$ we are looking for -- $T_{\!g}$ operates in the same input and output space as $T_{\!f}$, so $g\in\mathcal{R}_{\Sigma}$; furthermore, this function satisfies for all $\bm{x}\in\mathcal{X}\text{ and }i\in\left[1,n-1\right]\!:\,g(\bm{x})_i=\texttt{shift}(f(\bm{x}))_i=f(\bm{x})_{i+1}$.
\end{proof}
\vspace{0.75em}
\begin{corollary}
\Cref{theorem:closure} implies that when we parameterize a masked-diffusion model with a shifted transformer, it is as expressive as the standard non-shifted parameterization. More specifically, masked diffusion is defined in \Cref{eq:diffusion-elbo}, and $p_{\bm{\theta}}(x_i\mid\bm{x}^{t})$ is typically implemented as a fully-bidirectional transformer model that outputs this probability at the $i$th position. When we set $\Sigma$ as our subword vocabulary, we get that the space of all possible transformer realizations of $p_{\bm{\theta}}(x_i\mid\bm{x}^{t})$ are the $\Sigma$-realizable functions $\mathcal{R}_{\Sigma}$ (\Cref{def:realizable}). \Cref{theorem:closure} shows that if we instead expect the output at the $(i-1)$th position, we do not lose any expressivity. Thus, transformer-based dual-objective language models are a generalization of standard masked-diffusion language models. Note that the left-shift closure in \Cref{theorem:closure} works up to the first token -- which is guaranteed to be the special \texttt{<s>} token in \Cref{def:realizable} as well as in the actual implementation.
\end{corollary}

\newpage

\section{Validation loss curves}
\label{app:validation}

While we focused on actual downstream performance in the main experiments, we also show the validation loss below to demonstrate the training dynamics.

The validation curves in \Cref{fig:loss-128x} focus on an extremely data-constrained scenario with 128 data repetitions. There, it is crucial to avoid overfitting, which can be achieved by increasing the proportion of masked diffusion during training. Note that the noise of some of the curves is only due to our implementation of measuring the validation loss -- the sample size can be too small when the proportion of the respective training objective is low.

\begin{figure}[h!]
    \centering
    \includegraphics[width=\linewidth]{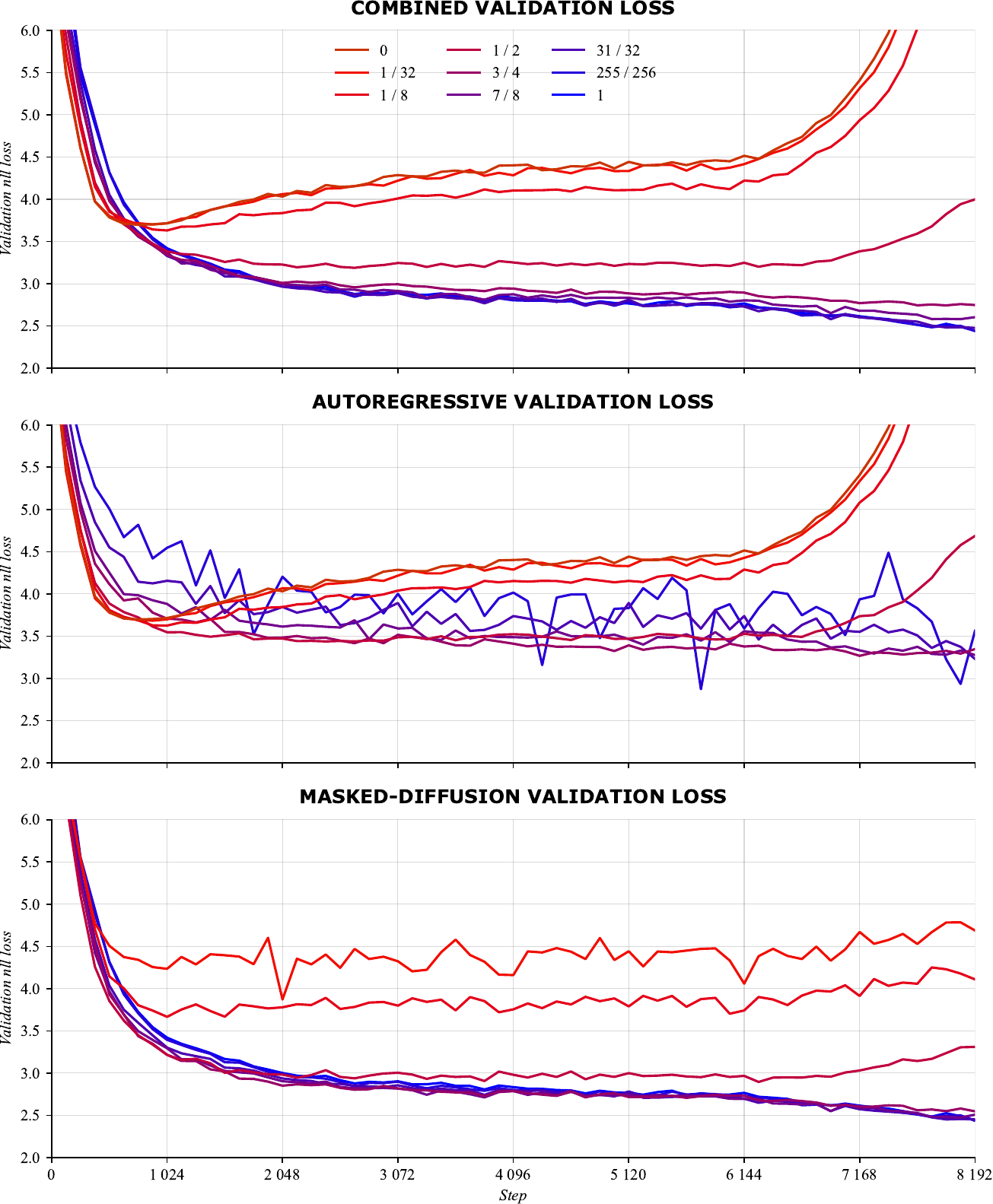}
    \caption{\textbf{Validation loss curves for 128 repetitions.} These plots clearly demonstrate how training runs with high ${\color{Red}\alpha}$ (in red) overfit. Low ${\color{Red}\alpha}$ values are in blue.}
    \label{fig:loss-128x}
\end{figure}

Contrary to the previous figure, \Cref{fig:loss-4x} shows validation curves for 4 data repetitions. Here, overfitting is not an issue; instead it is crucial to improve the learning speed by increasing the proportion of autoregressive language modeling.

\begin{figure}[h!]
    \centering
    \includegraphics[width=\linewidth]{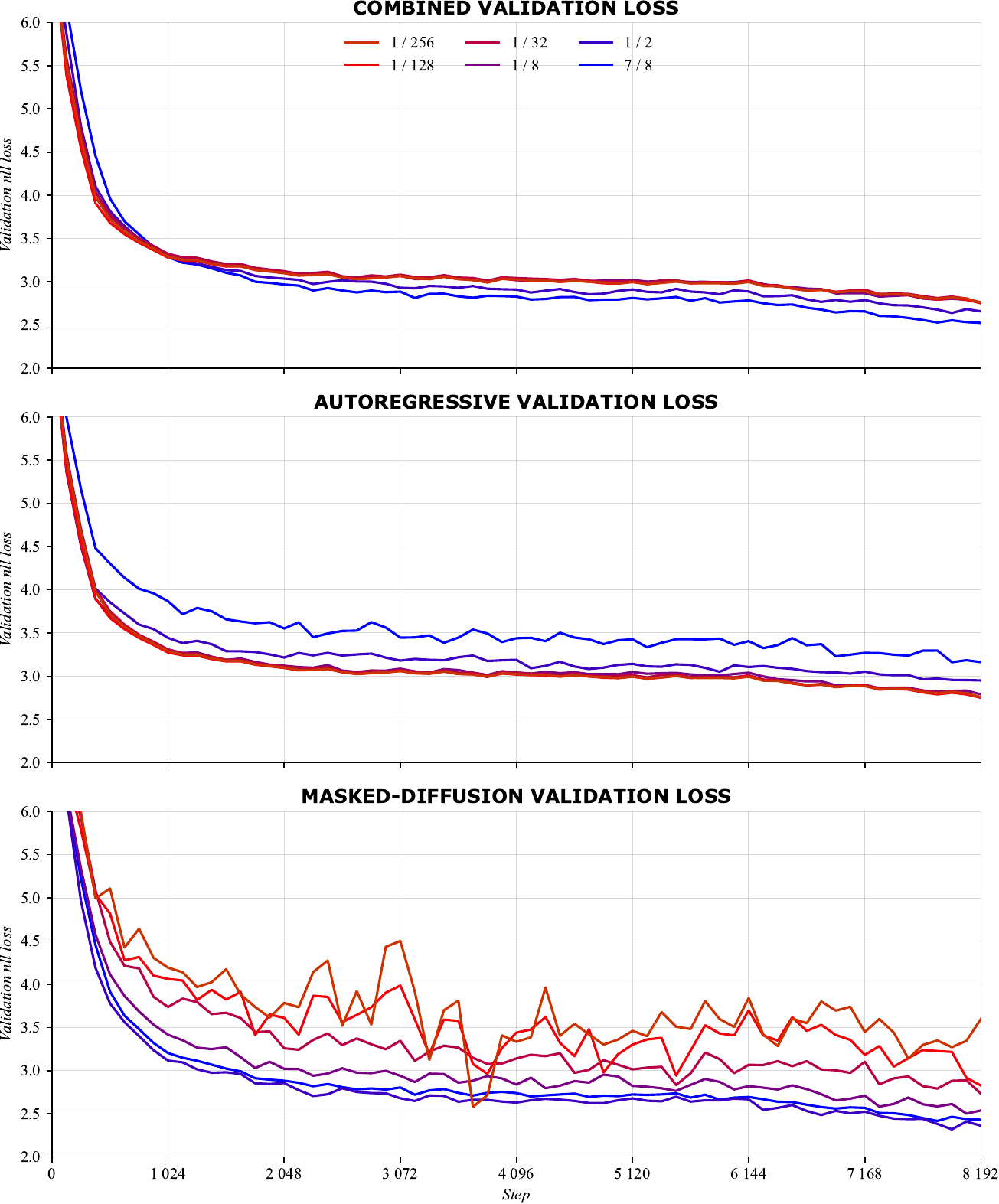}
    \caption{\textbf{Validation loss curves for 4 repetitions.} All losses monotonically decrease because overfitting is not a concern in this setting. High ${\color{Red}\alpha}$ values are plotted in red and low ${\color{Red}\alpha}$ values are shown in blue.}
    \label{fig:loss-4x}
\end{figure}

\section{Effects of number of mask tokens on the PLL}\label{app:num_tokens}

We first look at whether using a single number of mask tokens can lead to a good estimation of the PLL in general. For this, we evaluate five different models from 1 to 6 mask tokens and report the results in \cref{tab:model_a,tab:model_b,tab:model_c,tab:model_d,tab:model_e}.

\begin{table}[h]
    \centering
    \begin{minipage}[t]{0.48\textwidth}
        \centering
        \caption{\textbf{PLL performance depending on the number of mask tokens.} We show the PLL performance on the 9 tasks of the model trained with an equal ${\color{Red}\alpha}$ weight of masked-diffusion and AR and 32 repetitions with different number of masks. Best results per task are boldfaced.\vspace{1.4em}}
        \resizebox{\textwidth}{!}{%
        \begin{tabular}{@{}lcccccc@{}}
            \toprule
            \textbf{Task} & \multicolumn{6}{c}{\textbf{Number of masks}} \\
            \cline{2-7}
            & \textbf{1} & \textbf{2} & \textbf{3} & \textbf{4} & \textbf{5} & \textbf{6}\\
            \midrule
             ARC Easy&  18.7&  24.1&  25.5&  \textbf{26.3}&  26.0& \textbf{26.3}\\
             ARC Challenge&  \textbf{\hphantom{0}4.7}&  \hphantom{0}3.3&  \hphantom{0}3.8&  \hphantom{0}2.7&  \hphantom{0}1.9& \hphantom{0}2.6\\
             BLiMP&  \textbf{65.2}&  63.9&  62.5&  60.3&  60.5& 60.3\\
             Commonsense QA&  29.4&  32.8&  33.9&  \textbf{34.1}&  \textbf{34.1}& \textbf{34.1}\\
             HellaSwag&  \textbf{29.8}&  27.0&  26.7&  27.1&  26.7& 26.4\\
             MMLU&  \hphantom{0}2.0&  \textbf{\hphantom{0}3.5}&  \hphantom{0}3.1&  \hphantom{0}2.9&  \hphantom{0}3.3& \hphantom{0}3.3\\
             OpenBook QA&  \hphantom{0}9.1&  \hphantom{0}7.7&  \hphantom{0}8.5&  \textbf{\hphantom{0}9.3}&  \hphantom{0}7.2& \hphantom{0}6.9\\
             PIQA&  33.1&  34.3&  35.1&  35.4&  35.6& \textbf{36.8}\\
             SIQA&  11.4&  13.3&  13.7&  13.5&  \textbf{14.4}& \textbf{14.4}\\[0.3em]
             \textbf{Average} & 22.6 & 23.3 & \textbf{23.6} & 23.5 & 23.3 & 23.4\\
             \bottomrule
        \end{tabular}%
        }
        \label{tab:model_a}
    \end{minipage}%
    \hfill
    \begin{minipage}[t]{0.48\textwidth}
        \centering
        \caption{\textbf{PLL performance depending on the number of mask tokens.} We show the PLL performance on the 9 tasks of the model trained with a 1 masked-diffusion to 7 autoregressive ratio (${\color{Red}\alpha}=\nicefrac{7}{8}$) and 32 repetitions with different number of masks. Best results per task are boldfaced.\vspace{0.5em}}
        \resizebox{\textwidth}{!}{%
        \begin{tabular}{@{}lcccccc@{}}
            \toprule
            \textbf{Task} & \multicolumn{6}{c}{\textbf{Number of masks}} \\
            \cline{2-7}
            & \textbf{1} & \textbf{2} & \textbf{3} & \textbf{4} & \textbf{5} & \textbf{6}\\
            \midrule
             ARC Easy&  18.2&  25.6&  27.1&  28.2&  26.9& \textbf{27.5}\\
             ARC Challenge&  \hphantom{0}1.9&  \hphantom{0}3.2&  \hphantom{0}2.4&  \hphantom{0}2.6&  \hphantom{0}3.6& \textbf{\hphantom{0}4.7}\\
             BLiMP&  \textbf{61.2}&  60.0&  58.3&  56.9&  57.0& 57.3\\
             Commonsense QA&  24.2&  29.1&  29.0&  \textbf{29.4}&  \textbf{29.4}& \textbf{29.4}\\
             HellaSwag&  25.2&  25.7&  26.6&  27.0&  \textbf{26.8}& \textbf{26.8}\\
             MMLU&  \hphantom{0}1.9&  \hphantom{0}3.4&  \hphantom{0}4.0&  \hphantom{0}3.9&  \hphantom{0}4.0& \textbf{\hphantom{0}4.2}\\
             OpenBook QA&  \hphantom{0}9.9&  10.1&  \textbf{12.3}&  10.9&  10.1& \hphantom{0}9.6\\
             PIQA&  31.0&  34.7&  \textbf{36.1}&  36.0&  35.0& 35.9\\
             SIQA&  11.7&  11.8&  14.2&  13.7&  14.1& \textbf{14.3}\\[0.3em]
             \textbf{Average} & 20.6& 22.6& \textbf{23.3}& 23.2& 23.0& \textbf{23.3}\\
             \bottomrule
        \end{tabular}%
        }
        \label{tab:model_b}
    \end{minipage}
\end{table}

\begin{table}[h]
    \centering
    \begin{minipage}[t]{0.48\textwidth}
        \centering
        \caption{\textbf{PLL performance depending on the number of mask tokens.} We show the PLL performance on the 9 tasks of the model trained with a 7 masked-diffusion to 1 autoregressive ratio (${\color{Red}\alpha}=\nicefrac{1}{8}$) and 32 repetitions with different number of masks. Best results per task are boldfaced.\vspace{0.5em}}
        \resizebox{\textwidth}{!}{%
        \begin{tabular}{@{}lcccccc@{}}
            \toprule
            \textbf{Task} & \multicolumn{6}{c}{\textbf{Number of masks}} \\
            \cline{2-7}
            & \textbf{1} & \textbf{2} & \textbf{3} & \textbf{4} & \textbf{5} & \textbf{6}\\
            \midrule
             ARC Easy&  16.3&  20.8&  23.9&  24.0&  \textbf{24.9}& \textbf{24.9}\\
             ARC Challenge&  \textbf{\hphantom{0}5.7}&  \hphantom{0}3.9&  \hphantom{0}3.5&  \hphantom{0}1.8&  \hphantom{0}3.3& \hphantom{0}2.2\\
             BLiMP&  \textbf{69.5}&  67.6&  64.0&  60.7&  60.1& 60.1\\
             Commonsense QA&  25.4&  29.7&  30.6&  31.1&  31.1& \textbf{31.2}\\
             HellaSwag&  \textbf{25.5}&  22.8&  21.0&  21.2&  20.5& 19.8\\
             MMLU&  \hphantom{0}0.5&  \hphantom{0}2.2&  \hphantom{0}2.2&  \hphantom{0}2.0&  \textbf{\hphantom{0}2.5}& \hphantom{0}2.4\\
             OpenBook QA&  13.1&  12.0&  \textbf{15.2}&  14.4&  13.1& 13.9\\
             PIQA&  29.6&  30.3&  30.8&  30.1&  \textbf{31.2}& 31.0\\
             SIQA&  12.2&  15.0&  \textbf{15.2}&  13.6&  13.8& 13.9\\[0.3em]
             \textbf{Average} & 22.0& 22.7& \textbf{22.9}& 22.1& 22.3& 22.2\\
             \bottomrule
        \end{tabular}%
        }
        \label{tab:model_c}
    \end{minipage}%
    \hfill
    \begin{minipage}[t]{0.48\textwidth}
        \centering
        \caption{\textbf{PLL performance depending on the number of mask tokens.} We show the PLL performance on the 9 tasks of the model trained with an equal ratio of masked-diffusion and AR (${\color{Red}\alpha}=\nicefrac{1}{2}$) and 16 repetitions with different number of masks. Best results per task are boldfaced.\vspace{1.4em}}
        \resizebox{\textwidth}{!}{%
        \begin{tabular}{@{}lcccccc@{}}
            \toprule
            \textbf{Task} & \multicolumn{6}{c}{\textbf{Number of masks}} \\
            \cline{2-7}
            & \textbf{1} & \textbf{2} & \textbf{3} & \textbf{4} & \textbf{5} & \textbf{6}\\
            \midrule
             ARC Easy&  16.8&  23.7&  25.8&  25.8&  \textbf{26.1}& \textbf{26.1}\\
             ARC Challenge&  \textbf{\hphantom{0}7.2}&  \hphantom{0}4.4&  \hphantom{0}4.4&  \hphantom{0}4.8&  \hphantom{0}3.2& \hphantom{0}4.5\\
             BLiMP&  \textbf{65.3}&  64.8&  63.1&  60.7&  60.6& 60.4\\
             Commonsense QA&  29.7&  33.8&  35.1&  35.1&  \textbf{35.2}& \textbf{35.2}\\
             HellaSwag&  \textbf{30.5}&  27.9&  27.8&  27.9&  27.2& 26.8\\
             MMLU&  \hphantom{0}1.3&  \hphantom{0}2.4&  \textbf{\hphantom{0}2.9}&  \hphantom{0}2.5&  \hphantom{0}2.7& \hphantom{0}2.5\\
             OpenBook QA&  12.3&  12.0&  \textbf{13.1}&  11.2&  11.7& 11.7\\
             PIQA&  33.8&  34.6&  36.0&  34.7&  36.3& \textbf{37.0}\\
             SIQA&  14.3&  13.9&  15.9&  15.3&  15.9& \textbf{16.1}\\[0.3em]
             \textbf{Average} & 23.5& 24.2& \textbf{24.9}& 24.2& 24.3& 24.5\\
             \bottomrule
        \end{tabular}%
        }
        \label{tab:model_d}
    \end{minipage}
\end{table}

\begin{table}
    \centering
    \begin{minipage}[t]{0.5\textwidth}
    \caption{\textbf{PLL performance depending on the number of mask tokens.} We show the PLL performance on the 9 tasks of the model trained with an equal ratio of masked-diffusion and AR (${\color{Red}\alpha}=\nicefrac{1}{2}$) and 64 repetitions with different number of masks. Best results per task are boldfaced.\vspace{0.5em}}
    \centering
    \resizebox{\textwidth}{!}{%
    \begin{tabular}{@{}lcccccc@{}}
        \toprule
        \textbf{Task} & \multicolumn{6}{c}{\textbf{Number of masks}} \\
        \cline{2-7}
        & \textbf{1} & \textbf{2} & \textbf{3} & \textbf{4} & \textbf{5} & \textbf{6}\\
        \midrule
         ARC Easy&  16.6&  21.8&  23.5&  23.4&  23.1& 23.1\\
         ARC Challenge&  \hphantom{0}1.8&  \hphantom{0}3.9&  \hphantom{0}3.9&  \hphantom{0}3.2&  \textbf{\hphantom{0}4.0}& \hphantom{0}3.5\\
         BLiMP&  \textbf{63.1}&  61.2&  59.6&  57.5&  56.9& 56.9\\
         Commonsense QA&  24.6&  27.6&  28.5&  \textbf{28.7}&  \textbf{28.7}& \textbf{28.7}\\
         HellaSwag&  \textbf{26.8}&  25.2&  24.2&  24.7&  24.3& 24.1\\
         MMLU&  \hphantom{0}1.2&  \hphantom{0}3.1&  \hphantom{0}3.0&  \hphantom{0}3.2&  \textbf{\hphantom{0}3.4}& \hphantom{0}3.2\\
         OpenBook QA&  \hphantom{0}8.3&  \hphantom{0}8.5&  \textbf{11.7}&  10.1&  \hphantom{0}8.3& \hphantom{0}8.0\\
         PIQA&  31.0&  31.7&  32.1&  33.7&  \textbf{34.3}& 34.1\\
         SIQA&  \textbf{14.3}&  12.3&  \textbf{14.3}&  13.1&  13.3& 13.5\\[0.3em]
         \textbf{Average} & 20.8& 21.7& \textbf{22.3}& 22.0& 21.8& 21.7\\
         \bottomrule
    \end{tabular}%
    }
    \label{tab:model_e}
    \end{minipage}
\end{table}

\clearpage
We can see two clear trends from the results. The first is that the BLiMP and HellaSwag tasks are better evaluated with a single mask token, rather than multiple. This could be due to the simpler language found in these datasets. The second trend is that ARC-Easy, Commonsense QA, PIQA, and SIQA tend to do better with multi-token masking, this could be due to the more complex answers using more infrequent words that have a higher likelihood of being split into subwords. We therefore decide that using a combination of a single token mask for some tasks and a multiple tokens for others is the best solution. To find the optimal combination, we test all possible combinations. The results can be seen in \cref{tab:comb}.

\begin{table}[h!]
    \centering
    \caption{\textbf{PLL performance for combinations of one mask token and multi-mask token.} Best results per model are boldfaced.\vspace{0.5em}}
    \begin{tabular}{@{}l@{\hspace{2em}}ccccc@{}}
        \toprule
         \multirow{2}{*}{\textbf{Repetitions -- $\color{red}\alpha$}}&  \multicolumn{5}{c}{\textbf{Mask combination}}\\
         &  \textbf{1 -- 2}&  \textbf{1 -- 3}&  \textbf{1 -- 4}&  \textbf{1 -- 5}& \textbf{1 -- 6}\\
         \midrule
         32 -- \nicefrac{1}{2}&  24.1&  24.5&  24.6&  24.7& \textbf{24.8}\\
         32 -- \nicefrac{7}{8}&  22.8&  23.6&  23.7&  23.5& \textbf{23.8}\\
         32  -- \nicefrac{1}{8}&  23.5&  \textbf{24.3}&  24.0&  24.1& 24.2\\
         16 -- \nicefrac{1}{2}&  24.9&  25.7&  25.4&  25.7& \textbf{25.8}\\
         64 -- \nicefrac{1}{2}&  22.3&  \textbf{23.0}&  22.9&  22.9& 22.8\\ 
         \bottomrule
    \end{tabular}
    \label{tab:comb}
\end{table}

Based on \cref{tab:comb}, we decide to evaluate PLL for all models with both a single mask token and six mask tokens. Then we take the max performance between the two for each task.

\section{Total compute resources used for training}

The training of all 50 language models used in this paper was conducted on the LUMI supercomputer, each language model was trained on 128 AMD MI250X GPUs (which is equivalent to 256 logical devices) using roughly 1\,500 GPU hours. In total, the resources required for conducting all training runs equals to approximately 75\,000 GPU hours.

\section{PLL versus masked diffusion}\label{app:mntp_diff}

\cref{tab:pll_diff} shows that the performance of the masked-diffusion model is in general lower than that of the combined (1 and 6 mask) PLL. In addition, the two PLL evaluations took about 2 hours to complete while the masked-diffusion evaluation takes 12 hours to complete on a MI250X AMD GPU.

\begin{table}[h!]
    \centering
    \caption{\textbf{Normalized PLL versus Masked-Diffusion evaluation.} The scores for each task are normalized so that 0\% corresponds to the random baseline and 100\% is the perfect score. The best result for each task is in boldfaced. We evaluate a model trained with equal AR and masked-diffusion ratio (${\color{Red}\alpha}=\nicefrac{1}{2}$) and 32 repetitions. \vspace{0.5em}}
    \begin{tabular}{@{}lcc@{}}
        \toprule
         \textbf{Task}&  \textbf{PLL}& \textbf{Masked-Diffusion}\\
         \midrule
         ARC-Easy&  26.3& \textbf{27.1}\\
         BLiMP&  \textbf{65.2}& 56.5\\
         Commonsense QA&  \textbf{34.1}& 32.7\\
         HellaSwag&  \textbf{29.8}& 21.3\\
         PIQA&  \textbf{36.8}& 32.0\\
         \bottomrule
    \end{tabular}
    \label{tab:pll_diff}
\end{table}

\section{Prefix-LM versus autoregressive-LM on optimal models.}

\begin{table}[h]
    \centering
    \small
    \caption{\textbf{Normalized autoregressive and prefix performance of selected models.} The scores for each task are normalized so that 0\% corresponds to the random baseline and 100\% is the perfect score. The best result for each dataset size is in boldfaced. The results for BLiMP are the same, since there is no context and the prefix evaluation defaults to the autoregressive one. The AR ratios for the models are 12.5\% for the 128 repetitions, 75\% for the 32 repetitions, and 98.4\% for the single repetition.\vspace{0.5em}}
    \begin{tabular}{@{}lcccccccccc@{}}
        \toprule
        \textbf{Model} & \rotatebox{90}{\textbf{ARC-C}} & \rotatebox{90}{\textbf{ARC-E}} & \rotatebox{90}{\textbf{BLiMP}} & \rotatebox{90}{\textbf{CSQA}} & \rotatebox{90}{\textbf{HSwag}} & \rotatebox{90}{\textbf{MMLU}} & \rotatebox{90}{\textbf{OBQA}} & \rotatebox{90}{\textbf{PIQA}} & \rotatebox{90}{\textbf{SIQA}} & \rotatebox{0}{\textbf{Average}} \\
        \midrule
        \scriptsize{\textsc{1 repetition}} \\
         
         \hspace{1em}Autoregressive & \hphantom{0}5.7 & 28.6 & \textbf{63.7} & 35.1 & 31.1 & \hphantom{0}4.9 & \textbf{17.6} & 40.9 & 14.3 & 26.9 \\
         
         \hspace{1em}Prefix & \textbf{\hphantom{0}6.5} & \textbf{31.0} & \textbf{63.7} & \textbf{40.0} & \textbf{31.2} & \textbf{\hphantom{0}4.5} & 16.5 & \textbf{42.1} & \textbf{15.2} & \textbf{27.9} \\[0.5em]

         \scriptsize{\textsc{32 repetitions}} \\
         
         \hspace{1em}Autoregressive & \hphantom{0}3.3 & 28.0 & \textbf{57.9} & 31.1 & 26.4 & \hphantom{0}3.6 & 14.4 & 36.1 & 14.64 & 23.9 \\
         
         \hspace{1em}Prefix & \textbf{\hphantom{0}6.3} & \textbf{28.9} & \textbf{57.9} & \textbf{33.1} & \textbf{27.1} & \textbf{\hphantom{0}4.3} & \textbf{15.2} & \textbf{36.7} & \textbf{15.4} & \textbf{25.0} \\[0.5em]
         
        \scriptsize{\textsc{128 repetitions}} \\
        
         \hspace{1em}Autoregressive & \textbf{\hphantom{0}1.7} & 23.6 & \textbf{56.1} & 24.8 & 14.2 & \hphantom{0}1.6 & \hphantom{0}8.5 & 28.1 & 13.3 & 19.1 \\
         
         \hspace{1em}Prefix & \hphantom{0}1.3 & \textbf{24.1} & \textbf{56.1} & \textbf{28.5} & \textbf{12.4} & \textbf{\hphantom{0}2.3} & \textbf{10.9} & \textbf{30.9} & \textbf{15.2} & \textbf{20.5} \\
         \bottomrule
    \end{tabular}
    \label{tab:prefix_perf}
\end{table}

\cref{tab:prefix_perf} shows that evaluating with the prefix mask almost always outperforms using the causal mask when the models are optimally trained. This is true in both the regular and constrained data settings.

\section{Detailed results of diffusion-masked evaluation}\label{app:mntp_perf}

\renewcommand{\arraystretch}{1.0}
\begin{table}[h]
    \centering
    \caption{\textbf{The normalized PLL performance of selected models.} We show the results on all nine evaluated tasks for three repetition values; each repetition group contains the results of the best-performing ${\color{Red}\alpha}$ and of the autoregressive-only model. The scores for each task are normalized so that 0\% corresponds to random baseline and 100\% is the perfect score. The best result for each dataset size is boldfaced.\vspace{1em}}
    \resizebox{\textwidth}{!}{%
    \begin{tabular}{@{}lccccccccc@{\hspace{1.8em}}c@{}}
        \toprule
        \raisebox{0em}{\textbf{Model configuration}} & \rotatebox{90}{\textbf{ARC-C}} & \rotatebox{90}{\textbf{ARC-E}} & \rotatebox{90}{\textbf{BLiMP}} & \rotatebox{90}{\textbf{CSQA}} & \rotatebox{90}{\textbf{HSwag}} & \rotatebox{90}{\textbf{MMLU}} & \rotatebox{90}{\textbf{OBQA}} & \rotatebox{90}{\textbf{PIQA}} & \rotatebox{90}{\textbf{SIQA}} & \rotatebox{0}{\textbf{Average}} \\
        \midrule

         \scriptsize{\textsc{32 repetitions}} \\
         \hspace{1em}Dual (${\color{Red}\alpha}=\nicefrac{3}{4}$) & \textbf{\hphantom{0}6.0} & \textbf{28.3} & 62.7 & \textbf{33.4} & \textbf{27.8} & \textbf{\hphantom{0}4.3} & \textbf{12.3} & \textbf{37.4} & \textbf{15.4} & \textbf{25.3} \\
         \hspace{1em}Masked-Diffusion (${\color{Red}\alpha}=0$) & \,-0.1 & 22.3 & \textbf{64.8} & 29.0 & 24.1 & \hphantom{0}1.6 & \hphantom{0}9.1 & 27.2 & 14.4 & 21.4 \\[0.5em]

        \scriptsize{\textsc{128 repetitions}} \\
         \hspace{1em}Dual (${\color{Red}\alpha}=\nicefrac{1}{8}$) & \hphantom{0}2.8 & \textbf{23.3} & \textbf{63.5} & \textbf{30.5} & \textbf{25.0} & \hphantom{0}2.1 & \textbf{12.8} & \textbf{31.8} & \textbf{15.2} & \textbf{23.0} \\
         \hspace{1em}Masked-Diffusion (${\color{Red}\alpha}=0$) & \textbf{\hphantom{0}3.3} & 19.2 & 63.3 & 29.2 & 22.1 & \textbf{\hphantom{0}2.6} & \hphantom{0}9.3 & 28.3 & 12.0 & 21.0 \\
         
         \bottomrule
    \end{tabular}%
    }
    \label{tab:mntp_performance}
\end{table}

\cref{tab:mntp_performance} shows similar trends to those found in \cref{tab:performance}. The notable exception being for BLiMP where the performances are similar between both models. Unlike the autoregressive models, the performance of the purely masked-diffusion models is similar to each other. This is partially due to the model not overfitting, but also to it not being sample efficient. On the other hand we see that for the dual-objective models, the performance significantly increases as we increase the training data set size.

\end{document}

%% file: math_commands.tex
\usepackage{amsmath,amsfonts,bm}

\def\eqref#1{equation~\ref{#1}}

\def\1{\bm{1}}

\DeclareMathAlphabet{\mathsfit}{\encodingdefault}{\sfdefault}{m}{sl}
\SetMathAlphabet{\mathsfit}{bold}{\encodingdefault}{\sfdefault}{bx}{n}



%% file: iclr2026_conference.bib
@inproceedings{
gong2025scaling,
title={Scaling Diffusion Language Models via Adaptation from Autoregressive Models},
author={Shansan Gong and Shivam Agarwal and Yizhe Zhang and Jiacheng Ye and Lin Zheng and Mukai Li and Chenxin An and Peilin Zhao and Wei Bi and Jiawei Han and Hao Peng and Lingpeng Kong},
booktitle={The Thirteenth International Conference on Learning Representations},
year={2025},
url={https://openreview.net/forum?id=j1tSLYKwg8}
}

@inproceedings{lv-etal-2024-analysis,
    title = "An Analysis and Mitigation of the Reversal Curse",
    author = "Lv, Ang  and
      Zhang, Kaiyi  and
      Xie, Shufang  and
      Tu, Quan  and
      Chen, Yuhan  and
      Wen, Ji-Rong  and
      Yan, Rui",
    editor = "Al-Onaizan, Yaser  and
      Bansal, Mohit  and
      Chen, Yun-Nung",
    booktitle = "Proceedings of the 2024 Conference on Empirical Methods in Natural Language Processing",
    month = nov,
    year = "2024",
    address = "Miami, Florida, USA",
    publisher = "Association for Computational Linguistics",
    url = "https://aclanthology.org/2024.emnlp-main.754/",
    doi = "10.18653/v1/2024.emnlp-main.754",
    pages = "13603--13615"
}

@misc{ye2025dream7bdiffusionlarge,
      title={Dream 7{B}: {D}iffusion Large Language Models}, 
      author={Jiacheng Ye and Zhihui Xie and Lin Zheng and Jiahui Gao and Zirui Wu and Xin Jiang and Zhenguo Li and Lingpeng Kong},
      year={2025},
      eprint={2508.15487},
      archivePrefix={arXiv},
      primaryClass={cs.CL},
      url={https://arxiv.org/abs/2508.15487}, 
}

@misc{
ramachandran2018searching,
title={Searching for Activation Functions},
author={Prajit Ramachandran and Barret Zoph and Quoc V. Le},
year={2018},
url={https://openreview.net/forum?id=SkBYYyZRZ},
}

@misc{shazeer2020gluvariantsimprovetransformer,
      title={{GLU} Variants Improve Transformer}, 
      author={Noam Shazeer},
      year={2020},
      eprint={2002.05202},
      archivePrefix={arXiv},
      primaryClass={cs.LG},
      url={https://arxiv.org/abs/2002.05202}, 
}

@article{10.1016/j.neucom.2023.127063,
author = {Su, Jianlin and Ahmed, Murtadha and Lu, Yu and Pan, Shengfeng and Bo, Wen and Liu, Yunfeng},
title = {Ro{F}ormer: {E}nhanced transformer with Rotary Position Embedding},
year = {2024},
issue_date = {Feb 2024},
publisher = {Elsevier Science Publishers B. V.},
address = {NLD},
volume = {568},
number = {C},
issn = {0925-2312},
url = {https://doi.org/10.1016/j.neucom.2023.127063},
doi = {10.1016/j.neucom.2023.127063},
journal = {Neurocomput.},
month = feb,
numpages = {12}
}

@inbook{10.5555/3454287.3455397,
author = {Zhang, Biao and Sennrich, Rico},
title = {Root mean square layer normalization},
year = {2019},
publisher = {Curran Associates Inc.},
address = {Red Hook, NY, USA},
booktitle = {Proceedings of the 33rd International Conference on Neural Information Processing Systems},
articleno = {1110},
numpages = {12},
url = {https://dl.acm.org/doi/abs/10.5555/3454287.3455397}
}

@article{10.5555/3455716.3455856,
author = {Raffel, Colin and Shazeer, Noam and Roberts, Adam and Lee, Katherine and Narang, Sharan and Matena, Michael and Zhou, Yanqi and Li, Wei and Liu, Peter J.},
title = {Exploring the limits of transfer learning with a unified text-to-text transformer},
year = {2020},
issue_date = {January 2020},
publisher = {JMLR.org},
volume = {21},
number = {1},
issn = {1532-4435},
journal = {J. Mach. Learn. Res.},
month = {jan},
articleno = {140},
numpages = {67},
url={https://www.jmlr.org/papers/volume21/20-074/20-074.pdf}
}

@inproceedings{lewis-etal-2020-bart,
    title = "{BART}: Denoising Sequence-to-Sequence Pre-training for Natural Language Generation, Translation, and Comprehension",
    author = "Lewis, Mike  and
      Liu, Yinhan  and
      Goyal, Naman  and
      Ghazvininejad, Marjan  and
      Mohamed, Abdelrahman  and
      Levy, Omer  and
      Stoyanov, Veselin  and
      Zettlemoyer, Luke",
    editor = "Jurafsky, Dan  and
      Chai, Joyce  and
      Schluter, Natalie  and
      Tetreault, Joel",
    booktitle = "Proceedings of the 58th Annual Meeting of the Association for Computational Linguistics",
    month = jul,
    year = "2020",
    address = "Online",
    publisher = "Association for Computational Linguistics",
    url = "https://aclanthology.org/2020.acl-main.703/",
    doi = "10.18653/v1/2020.acl-main.703",
    pages = "7871--7880"
}

@inproceedings{du-etal-2022-glm,
    title = "{GLM}: General Language Model Pretraining with Autoregressive Blank Infilling",
    author = "Du, Zhengxiao  and
      Qian, Yujie  and
      Liu, Xiao  and
      Ding, Ming  and
      Qiu, Jiezhong  and
      Yang, Zhilin  and
      Tang, Jie",
    editor = "Muresan, Smaranda  and
      Nakov, Preslav  and
      Villavicencio, Aline",
    booktitle = "Proceedings of the 60th Annual Meeting of the Association for Computational Linguistics (Volume 1: Long Papers)",
    month = may,
    year = "2022",
    address = "Dublin, Ireland",
    publisher = "Association for Computational Linguistics",
    url = "https://aclanthology.org/2022.acl-long.26/",
    doi = "10.18653/v1/2022.acl-long.26",
    pages = "320--335"
}

@misc{aghajanyan2022cm3,
      title={{CM3}: A Causal Masked Multimodal Model of the Internet}, 
      author={Armen Aghajanyan and Bernie Huang and Candace Ross and Vladimir Karpukhin and Hu Xu and Naman Goyal and Dmytro Okhonko and Mandar Joshi and Gargi Ghosh and Mike Lewis and Luke Zettlemoyer},
      year={2022},
      eprint={2201.07520},
      archivePrefix={arXiv},
      primaryClass={cs.CL},
      url={https://arxiv.org/abs/2201.07520}
}

@inproceedings{NEURIPS2019_c20bb2d9,
 author = {Dong, Li and Yang, Nan and Wang, Wenhui and Wei, Furu and Liu, Xiaodong and Wang, Yu and Gao, Jianfeng and Zhou, Ming and Hon, Hsiao-Wuen},
 booktitle = {Advances in Neural Information Processing Systems},
 editor = {H. Wallach and H. Larochelle and A. Beygelzimer and F. d\textquotesingle Alch\'{e}-Buc and E. Fox and R. Garnett},
 pages = {},
 publisher = {Curran Associates, Inc.},
 title = {Unified Language Model Pre-training for Natural Language Understanding and Generation},
 url = {https://proceedings.neurips.cc/paper_files/paper/2019/file/c20bb2d9a50d5ac1f713f8b34d9aac5a-Paper.pdf},
 volume = {32},
 year = {2019}
}

@inproceedings{
tay2023ul,
title={{UL}2: Unifying Language Learning Paradigms},
author={Yi Tay and Mostafa Dehghani and Vinh Q. Tran and Xavier Garcia and Jason Wei and Xuezhi Wang and Hyung Won Chung and Dara Bahri and Tal Schuster and Steven Zheng and Denny Zhou and Neil Houlsby and Donald Metzler},
booktitle={The Eleventh International Conference on Learning Representations },
year={2023},
url={https://openreview.net/forum?id=6ruVLB727MC}
}

@misc{prabhudesai2025diffusionbeatsautoregressivedataconstrained,
      title={Diffusion Beats Autoregressive in Data-Constrained Settings}, 
      author={Mihir Prabhudesai and Mengning Wu and Amir Zadeh and Katerina Fragkiadaki and Deepak Pathak},
      year={2025},
      eprint={2507.15857},
      archivePrefix={arXiv},
      primaryClass={cs.LG},
      url={https://arxiv.org/abs/2507.15857}, 
}

@ARTICLE{2020SciPy-NMeth,
  author  = {Virtanen, Pauli and Gommers, Ralf and Oliphant, Travis E. and
            Haberland, Matt and Reddy, Tyler and Cournapeau, David and
            Burovski, Evgeni and Peterson, Pearu and Weckesser, Warren and
            Bright, Jonathan and {van der Walt}, St{\'e}fan J. and
            Brett, Matthew and Wilson, Joshua and Millman, K. Jarrod and
            Mayorov, Nikolay and Nelson, Andrew R. J. and Jones, Eric and
            Kern, Robert and Larson, Eric and Carey, C J and
            Polat, {\.I}lhan and Feng, Yu and Moore, Eric W. and
            {VanderPlas}, Jake and Laxalde, Denis and Perktold, Josef and
            Cimrman, Robert and Henriksen, Ian and Quintero, E. A. and
            Harris, Charles R. and Archibald, Anne M. and
            Ribeiro, Ant{\^o}nio H. and Pedregosa, Fabian and
            {van Mulbregt}, Paul and {SciPy 1.0 Contributors}},
  title   = {{{SciPy} 1.0: Fundamental Algorithms for Scientific
            Computing in Python}},
  journal = {Nature Methods},
  year    = {2020},
  volume  = {17},
  pages   = {261--272},
  adsurl  = {https://rdcu.be/b08Wh},
  doi     = {10.1038/s41592-019-0686-2},
  url     = {https://www.nature.com/articles/s41592-019-0686-2}
}

@article{10.5555/177910.177914,
author = {Gage, Philip},
title = {A new algorithm for data compression},
year = {1994},
issue_date = {Feb. 1994},
publisher = {R \& D Publications, Inc.},
address = {USA},
volume = {12},
number = {2},
issn = {0898-9788},
journal = {C Users J.},
month = feb,
pages = {23–38},
numpages = {16},
url = {https://dl.acm.org/doi/abs/10.5555/177910.177914}
}

@inproceedings{
nie2025scaling,
title={Scaling up Masked Diffusion Models on Text},
author={Shen Nie and Fengqi Zhu and Chao Du and Tianyu Pang and Qian Liu and Guangtao Zeng and Min Lin and Chongxuan Li},
booktitle={The Thirteenth International Conference on Learning Representations},
year={2025},
url={https://openreview.net/forum?id=WNvvwK0tut}
}

@misc{kaplan2020scalinglawsneurallanguage,
      title={Scaling Laws for Neural Language Models}, 
      author={Jared Kaplan and Sam McCandlish and Tom Henighan and Tom B. Brown and Benjamin Chess and Rewon Child and Scott Gray and Alec Radford and Jeffrey Wu and Dario Amodei},
      year={2020},
      eprint={2001.08361},
      archivePrefix={arXiv},
      primaryClass={cs.LG},
      url={https://arxiv.org/abs/2001.08361}, 
}

@inproceedings{10.5555/3692070.3694094,
    author = {Villalobos, Pablo and Ho, Anson and Sevilla, Jaime and Besiroglu, Tamay and Heim, Lennart and Hobbhahn, Marius},
    title = {Position: {W}ill we run out of data? {L}imits of LLM scaling based on human-generated data},
    year = {2024},
    publisher = {JMLR.org},
    booktitle = {Proceedings of the 41st International Conference on Machine Learning},
    articleno = {2024},
    numpages = {22},
    location = {Vienna, Austria},
    series = {ICML'24},
    url = {https://dl.acm.org/doi/10.5555/3692070.3694094}
    }

@inproceedings{NEURIPS2023_9d89448b,
 author = {Muennighoff, Niklas and Rush, Alexander and Barak, Boaz and Le Scao, Teven and Tazi, Nouamane and Piktus, Aleksandra and Pyysalo, Sampo and Wolf, Thomas and Raffel, Colin A},
 booktitle = {Advances in Neural Information Processing Systems},
 editor = {A. Oh and T. Naumann and A. Globerson and K. Saenko and M. Hardt and S. Levine},
 pages = {50358--50376},
 publisher = {Curran Associates, Inc.},
 title = {Scaling Data-Constrained Language Models},
 url = {https://proceedings.neurips.cc/paper_files/paper/2023/file/9d89448b63ce1e2e8dc7af72c984c196-Paper-Conference.pdf},
 volume = {36},
 year = {2023}
}

@inproceedings{NIPS1995_7cce53cf,
 author = {Williams, Christopher and Rasmussen, Carl},
 booktitle = {Advances in Neural Information Processing Systems},
 editor = {D. Touretzky and M.C. Mozer and M. Hasselmo},
 pages = {},
 publisher = {MIT Press},
 title = {Gaussian Processes for Regression},
 url = {https://proceedings.neurips.cc/paper_files/paper/1995/file/7cce53cf90577442771720a370c3c723-Paper.pdf},
 volume = {8},
 year = {1995}
}

@inproceedings{burchell-etal-2025-expanded,
    title = "An Expanded Massive Multilingual Dataset for High-Performance Language Technologies ({HPLT})",
    author = {Burchell, Laurie  and
      de Gibert, Ona  and
      Arefyev, Nikolay  and
      Aulamo, Mikko  and
      Ba{\~n}{\'o}n, Marta  and
      Chen, Pinzhen  and
      Fedorova, Mariia  and
      Guillou, Liane  and
      Haddow, Barry  and
      Haji{\v{c}}, Jan  and
      Helcl, Jind{\v{r}}ich  and
      Henriksson, Erik  and
      Klimaszewski, Mateusz  and
      Komulainen, Ville  and
      Kutuzov, Andrey  and
      Kyt{\"o}niemi, Joona  and
      Laippala, Veronika  and
      M{\ae}hlum, Petter  and
      Malik, Bhavitvya  and
      Mehryary, Farrokh  and
      Mikhailov, Vladislav  and
      Moghe, Nikita  and
      Myntti, Amanda  and
      O{'}Brien, Dayy{\'a}n  and
      Oepen, Stephan  and
      Pal, Proyag  and
      Piha, Jousia  and
      Pyysalo, Sampo  and
      Ram{\'i}rez-S{\'a}nchez, Gema  and
      Samuel, David  and
      Stepachev, Pavel  and
      Tiedemann, J{\"o}rg  and
      Vari{\v{s}}, Du{\v{s}}an  and
      Vojt{\v{e}}chov{\'a}, Tereza  and
      Zaragoza-Bernabeu, Jaume},
    editor = "Che, Wanxiang  and
      Nabende, Joyce  and
      Shutova, Ekaterina  and
      Pilehvar, Mohammad Taher",
    booktitle = "Proceedings of the 63rd Annual Meeting of the Association for Computational Linguistics (Volume 1: Long Papers)",
    month = jul,
    year = "2025",
    address = "Vienna, Austria",
    publisher = "Association for Computational Linguistics",
    url = "https://aclanthology.org/2025.acl-long.854/",
    doi = "10.18653/v1/2025.acl-long.854",
    pages = "17452--17485",
    ISBN = "979-8-89176-251-0"
}

@article{10.5555/3112655.3112866,
author = {Liu, Dong C. and Nocedal, Jorge},
title = {On the limited memory {BFGS} method for large scale optimization},
year = {1989},
issue_date = {August    1989},
publisher = {Springer-Verlag},
address = {Berlin, Heidelberg},
volume = {45},
number = {1–3},
issn = {0025-5610},
journal = {Math. Program.},
month = aug,
pages = {503–528},
numpages = {26},
url={https://link.springer.com/article/10.1007/BF01589116}
}

@book{stein1999interpolation,
  address = {New York},
  author = {Stein, Michael L.},
  description = {MR: Publications results for "MR Number=(1697409)"},
  doi = {10.1007/978-1-4612-1494-6},
  isbn = {0-387-98629-4},
  mrclass = {62M30 (62-02 62M40)},
  mrnumber = {1697409 (2000f:62236)},
  mrreviewer = {Dietrich Stoyan},
  note = {Some theory for Kriging},
  pages = {xviii+247},
  publisher = {Springer-Verlag},
  series = {Springer Series in Statistics},
  title = {Interpolation of spatial data},
  url = {http://dx.doi.org/10.1007/978-1-4612-1494-6},
  year = 1999
}

@misc{ni2025difflm,
      title={Diffusion Language Models are Super Data Learners}, 
      author={Jinjie Ni and Qian Liu and Longxu Dou and Chao Du and Zili Wang and Hang Yan and Tianyu Pang and Michael Qizhe Shieh},
      year={2025},
      eprint={2511.03276},
      archivePrefix={arXiv},
      primaryClass={cs.LG},
      url={https://arxiv.org/abs/2511.03276}, 
}

@misc{chowdhery2022palmscalinglanguagemodeling,
      title={Pa{LM}: Scaling Language Modeling with Pathways}, 
      author={Aakanksha Chowdhery and Sharan Narang and Jacob Devlin and Maarten Bosma and Gaurav Mishra and Adam Roberts and Paul Barham and Hyung Won Chung and Charles Sutton and Sebastian Gehrmann and Parker Schuh and Kensen Shi and Sasha Tsvyashchenko and Joshua Maynez and Abhishek Rao and Parker Barnes and Yi Tay and Noam Shazeer and Vinodkumar Prabhakaran and Emily Reif and Nan Du and Ben Hutchinson and Reiner Pope and James Bradbury and Jacob Austin and Michael Isard and Guy Gur-Ari and Pengcheng Yin and Toju Duke and Anselm Levskaya and Sanjay Ghemawat and Sunipa Dev and Henryk Michalewski and Xavier Garcia and Vedant Misra and Kevin Robinson and Liam Fedus and Denny Zhou and Daphne Ippolito and David Luan and Hyeontaek Lim and Barret Zoph and Alexander Spiridonov and Ryan Sepassi and David Dohan and Shivani Agrawal and Mark Omernick and Andrew M. Dai and Thanumalayan Sankaranarayana Pillai and Marie Pellat and Aitor Lewkowycz and Erica Moreira and Rewon Child and Oleksandr Polozov and Katherine Lee and Zongwei Zhou and Xuezhi Wang and Brennan Saeta and Mark Diaz and Orhan Firat and Michele Catasta and Jason Wei and Kathy Meier-Hellstern and Douglas Eck and Jeff Dean and Slav Petrov and Noah Fiedel},
      year={2022},
      eprint={2204.02311},
      archivePrefix={arXiv},
      primaryClass={cs.CL},
      url={https://arxiv.org/abs/2204.02311}, 
}

@inproceedings{
hagele2024scaling,
title={Scaling Laws and Compute-Optimal Training Beyond Fixed Training Durations},
author={Alexander H{\"a}gele and Elie Bakouch and Atli Kosson and Loubna Ben allal and Leandro Von Werra and Martin Jaggi},
booktitle={Workshop on Efficient Systems for Foundation Models II @ ICML2024},
year={2024},
url={https://openreview.net/forum?id=ompl7supoX}
}

@inproceedings{nguyen-salazar-2019-transformers,
    title = "Transformers without Tears: {I}mproving the Normalization of Self-Attention",
    author = "Nguyen, Toan Q.  and
      Salazar, Julian",
    editor = {Niehues, Jan  and
      Cattoni, Rolando  and
      St{\"u}ker, Sebastian  and
      Negri, Matteo  and
      Turchi, Marco  and
      Ha, Thanh-Le  and
      Salesky, Elizabeth  and
      Sanabria, Ramon  and
      Barrault, Loic  and
      Specia, Lucia  and
      Federico, Marcello},
    booktitle = "Proceedings of the 16th International Conference on Spoken Language Translation",
    month = nov # " 2-3",
    year = "2019",
    address = "Hong Kong",
    publisher = "Association for Computational Linguistics",
    url = "https://aclanthology.org/2019.iwslt-1.17/"
}

@misc{liu2025muonscalablellmtraining,
      title={Muon is Scalable for {LLM} Training}, 
      author={Jingyuan Liu and Jianlin Su and Xingcheng Yao and Zhejun Jiang and Guokun Lai and Yulun Du and Yidao Qin and Weixin Xu and Enzhe Lu and Junjie Yan and Yanru Chen and Huabin Zheng and Yibo Liu and Shaowei Liu and Bohong Yin and Weiran He and Han Zhu and Yuzhi Wang and Jianzhou Wang and Mengnan Dong and Zheng Zhang and Yongsheng Kang and Hao Zhang and Xinran Xu and Yutao Zhang and Yuxin Wu and Xinyu Zhou and Zhilin Yang},
      year={2025},
      eprint={2502.16982},
      archivePrefix={arXiv},
      primaryClass={cs.LG},
      url={https://arxiv.org/abs/2502.16982}, 
}

@misc{jordan2024muon,
  author       = {Keller Jordan and Yuchen Jin and Vlado Boza and Jiacheng You and
                  Franz Cesista and Laker Newhouse and Jeremy Bernstein},
  title        = {Muon: {A}n optimizer for hidden layers in neural networks},
  year         = {2024},
  url          = {https://kellerjordan.github.io/posts/muon/}
}

@inproceedings{
hendrycks2021measuring,
title={Measuring Massive Multitask Language Understanding},
author={Dan Hendrycks and Collin Burns and Steven Basart and Andy Zou and Mantas Mazeika and Dawn Song and Jacob Steinhardt},
booktitle={International Conference on Learning Representations},
year={2021},
url={https://openreview.net/forum?id=d7KBjmI3GmQ}
}

@inproceedings{10.5555/3600270.3602446,
author = {Hoffmann, Jordan and Borgeaud, Sebastian and Mensch, Arthur and Buchatskaya, Elena and Cai, Trevor and Rutherford, Eliza and de Las Casas, Diego and Hendricks, Lisa Anne and Welbl, Johannes and Clark, Aidan and Hennigan, Tom and Noland, Eric and Millican, Katie and van den Driessche, George and Damoc, Bogdan and Guy, Aurelia and Osindero, Simon and Simonyan, Karen and Elsen, Erich and Vinyals, Oriol and Rae, Jack W. and Sifre, Laurent},
title = {Training compute-optimal large language models},
year = {2022},
isbn = {9781713871088},
publisher = {Curran Associates Inc.},
address = {Red Hook, NY, USA},
booktitle = {Proceedings of the 36th International Conference on Neural Information Processing Systems},
articleno = {2176},
numpages = {15},
location = {New Orleans, LA, USA},
series = {NIPS '22},
url = {https://dl.acm.org/doi/10.5555/3600270.3602446}
}

@proceedings{conll-2024-babylm,
    title = "The 2nd BabyLM Challenge at the 28th Conference on Computational Natural Language Learning",
    editor = "Hu, Michael Y.  and
      Mueller, Aaron  and
      Ross, Candace  and
      Williams, Adina  and
      Linzen, Tal  and
      Zhuang, Chengxu  and
      Choshen, Leshem  and
      Cotterell, Ryan  and
      Warstadt, Alex  and
      Wilcox, Ethan Gotlieb",
    month = nov,
    year = "2024",
    address = "Miami, FL, USA",
    publisher = "Association for Computational Linguistics",
    url = "https://aclanthology.org/2024.conll-babylm.0/"
}

@inproceedings{charpentier-samuel-2024-bert,
    title = "{GPT} or {BERT}: {W}hy not both?",
    author = "Charpentier, Lucas Georges Gabriel  and
      Samuel, David",
    editor = "Hu, Michael Y.  and
      Mueller, Aaron  and
      Ross, Candace  and
      Williams, Adina  and
      Linzen, Tal  and
      Zhuang, Chengxu  and
      Choshen, Leshem  and
      Cotterell, Ryan  and
      Warstadt, Alex  and
      Wilcox, Ethan Gotlieb",
    booktitle = "The 2nd BabyLM Challenge at the 28th Conference on Computational Natural Language Learning",
    month = nov,
    year = "2024",
    address = "Miami, FL, USA",
    publisher = "Association for Computational Linguistics",
    url = "https://aclanthology.org/2024.conll-babylm.24/",
    pages = "262--283"
}

@inproceedings{NIPS2017_3f5ee243,
 author = {Vaswani, Ashish and Shazeer, Noam and Parmar, Niki and Uszkoreit, Jakob and Jones, Llion and Gomez, Aidan N and Kaiser, \L ukasz and Polosukhin, Illia},
 booktitle = {Advances in Neural Information Processing Systems},
 editor = {I. Guyon and U. Von Luxburg and S. Bengio and H. Wallach and R. Fergus and S. Vishwanathan and R. Garnett},
 pages = {},
 publisher = {Curran Associates, Inc.},
 title = {Attention is All you Need},
 url = {https://proceedings.neurips.cc/paper_files/paper/2017/file/3f5ee243547dee91fbd053c1c4a845aa-Paper.pdf},
 volume = {30},
 year = {2017}
}

@inproceedings{10.5555/3737916.3742051,
author = {Sahoo, Subham Sekhar and Arriola, Marianne and Schiff, Yair and Gokaslan, Aaron and Marroquin, Edgar and Chiu, Justin T and Rush, Alexander and Kuleshov, Volodymyr},
title = {Simple and effective masked diffusion language models},
year = {2025},
isbn = {9798331314385},
publisher = {Curran Associates Inc.},
address = {Red Hook, NY, USA},
booktitle = {Proceedings of the 38th International Conference on Neural Information Processing Systems},
articleno = {4135},
numpages = {49},
location = {Vancouver, BC, Canada},
series = {NIPS '24},
url={https://dl.acm.org/doi/10.5555/3737916.3742051}
}

@inproceedings{10.5555/3692070.3693403,
author = {Lou, Aaron and Meng, Chenlin and Ermon, Stefano},
title = {Discrete diffusion modeling by estimating the ratios of the data distribution},
year = {2024},
publisher = {JMLR.org},
booktitle = {Proceedings of the 41st International Conference on Machine Learning},
articleno = {1333},
numpages = {30},
location = {Vienna, Austria},
series = {ICML'24},
url={https://dl.acm.org/doi/10.5555/3692070.3693403}
}

@inproceedings{
nie2025large,
title={Large Language Diffusion Models},
author={Shen Nie and Fengqi Zhu and Zebin You and Xiaolu Zhang and Jingyang Ou and Jun Hu and Jun Zhou and Yankai Lin and Ji-Rong Wen and Chongxuan Li},
booktitle={ICLR 2025 Workshop on Deep Generative Model in Machine Learning: Theory, Principle and Efficacy},
year={2025},
url={https://openreview.net/forum?id=wzl61tIUj6}
}

@article{metropolis1949monte,
  title={The {M}onte {C}arlo Method},
  author={Metropolis, Nicholas and Ulam, S.},
  journal={Journal of the American Statistical Association},
  volume={44},
  number={247},
  pages={335--341},
  year={1949},
  publisher={Taylor \& Francis},
  doi={10.1080/01621459.1949.10483310},
  url={https://doi.org/10.1080/01621459.1949.10483310}
}

@inproceedings{
berglund2024the,
title={The Reversal Curse: {LLM}s trained on {\textquotedblleft}A is B{\textquotedblright} fail to learn {\textquotedblleft}B is A{\textquotedblright}},
author={Lukas Berglund and Meg Tong and Maximilian Kaufmann and Mikita Balesni and Asa Cooper Stickland and Tomasz Korbak and Owain Evans},
booktitle={The Twelfth International Conference on Learning Representations},
year={2024},
url={https://openreview.net/forum?id=GPKTIktA0k}
}

@inproceedings{10.5555/3737916.3738000,
author = {Samuel, David},
title = {{BERT}s are generative in-context learners},
year = {2025},
isbn = {9798331314385},
publisher = {Curran Associates Inc.},
address = {Red Hook, NY, USA},
booktitle = {Proceedings of the 38th International Conference on Neural Information Processing Systems},
articleno = {84},
numpages = {32},
location = {Vancouver, BC, Canada},
series = {NIPS '24},
url={https://dl.acm.org/doi/10.5555/3737916.3738000}
}

@inproceedings{NEURIPS2020_1457c0d6,
 author = {Brown, Tom and Mann, Benjamin and Ryder, Nick and Subbiah, Melanie and Kaplan, Jared D and Dhariwal, Prafulla and Neelakantan, Arvind and Shyam, Pranav and Sastry, Girish and Askell, Amanda and Agarwal, Sandhini and Herbert-Voss, Ariel and Krueger, Gretchen and Henighan, Tom and Child, Rewon and Ramesh, Aditya and Ziegler, Daniel and Wu, Jeffrey and Winter, Clemens and Hesse, Chris and Chen, Mark and Sigler, Eric and Litwin, Mateusz and Gray, Scott and Chess, Benjamin and Clark, Jack and Berner, Christopher and McCandlish, Sam and Radford, Alec and Sutskever, Ilya and Amodei, Dario},
 booktitle = {Advances in Neural Information Processing Systems},
 editor = {H. Larochelle and M. Ranzato and R. Hadsell and M.F. Balcan and H. Lin},
 pages = {1877--1901},
 publisher = {Curran Associates, Inc.},
 title = {Language Models are Few-Shot Learners},
 url = {https://proceedings.neurips.cc/paper_files/paper/2020/file/1457c0d6bfcb4967418bfb8ac142f64a-Paper.pdf},
 volume = {33},
 year = {2020}
}

@inproceedings{
ou2025your,
title={Your Absorbing Discrete Diffusion Secretly Models the Conditional Distributions of Clean Data},
author={Jingyang Ou and Shen Nie and Kaiwen Xue and Fengqi Zhu and Jiacheng Sun and Zhenguo Li and Chongxuan Li},
booktitle={The Thirteenth International Conference on Learning Representations},
year={2025},
url={https://openreview.net/forum?id=sMyXP8Tanm}
}

@article{fisher1922mathematical,
  title={On the mathematical foundations of theoretical statistics},
  author={Fisher, R. A.},
  journal={Philosophical Transactions of the Royal Society of London. Series A, Containing Papers of a Mathematical or Physical Character},
  volume={222},
  pages={309--368},
  year={1922},
  publisher={The Royal Society London},
  doi={10.1098/rsta.1922.0009},
  url={https://royalsocietypublishing.org/doi/10.1098/rsta.1922.0009}
}

@article{shannon1951prediction,
  title={Prediction and Entropy of Printed {E}nglish},
  author={Shannon, Claude E.},
  journal={Bell System Technical Journal},
  volume={30},
  number={1},
  pages={50--64},
  year={1951},
  month={January},
  doi={10.1002/j.1538-7305.1951.tb01366.x},
  url={https://onlinelibrary.wiley.com/doi/abs/10.1002/j.1538-7305.1951.tb01366.x}
}

@article{fisher1925theory,
  title={Theory of statistical estimation},
  author={Fisher, R. A.},
  journal={Mathematical Proceedings of the Cambridge Philosophical Society},
  volume={22},
  number={5},
  pages={700--725},
  year={1925},
  publisher={Cambridge University Press},
  doi={10.1017/S0305004100009580},
  url={https://www.cambridge.org/core/journals/mathematical-proceedings-of-the-cambridge-philosophical-society/article/abs/theory-of-statistical-estimation/7A05FB68C83B36C0E91D42C76AB177D4}
}

@inproceedings{devlin-etal-2019-bert,
    title = "{BERT}: {P}re-training of Deep Bidirectional Transformers for Language Understanding",
    author = "Devlin, Jacob  and
      Chang, Ming-Wei  and
      Lee, Kenton  and
      Toutanova, Kristina",
    editor = "Burstein, Jill  and
      Doran, Christy  and
      Solorio, Thamar",
    booktitle = "Proceedings of the 2019 Conference of the North {A}merican Chapter of the Association for Computational Linguistics: Human Language Technologies, Volume 1 (Long and Short Papers)",
    month = jun,
    year = "2019",
    address = "Minneapolis, Minnesota",
    publisher = "Association for Computational Linguistics",
    url = "https://aclanthology.org/N19-1423/",
    doi = "10.18653/v1/N19-1423",
    pages = "4171--4186"
}

@inproceedings{NEURIPS2021_958c5305,
 author = {Austin, Jacob and Johnson, Daniel D. and Ho, Jonathan and Tarlow, Daniel and van den Berg, Rianne},
 booktitle = {Advances in Neural Information Processing Systems},
 editor = {M. Ranzato and A. Beygelzimer and Y. Dauphin and P.S. Liang and J. Wortman Vaughan},
 pages = {17981--17993},
 publisher = {Curran Associates, Inc.},
 title = {Structured Denoising Diffusion Models in Discrete State-Spaces},
 url = {https://proceedings.neurips.cc/paper_files/paper/2021/file/958c530554f78bcd8e97125b70e6973d-Paper.pdf},
 volume = {34},
 year = {2021}
}

@inproceedings{wang-cho-2019-bert,
    title = "{BERT} has a Mouth, and It Must Speak: {BERT} as a {M}arkov Random Field Language Model",
    author = "Wang, Alex  and
      Cho, Kyunghyun",
    editor = "Bosselut, Antoine  and
      Celikyilmaz, Asli  and
      Ghazvininejad, Marjan  and
      Iyer, Srinivasan  and
      Khandelwal, Urvashi  and
      Rashkin, Hannah  and
      Wolf, Thomas",
    booktitle = "Proceedings of the Workshop on Methods for Optimizing and Evaluating Neural Language Generation",
    month = jun,
    year = "2019",
    address = "Minneapolis, Minnesota",
    publisher = "Association for Computational Linguistics",
    url = "https://aclanthology.org/W19-2304/",
    doi = "10.18653/v1/W19-2304",
    pages = "30--36"
}

@inproceedings{salazar-etal-2020-masked,
    title = "Masked Language Model Scoring",
    author = "Salazar, Julian  and
      Liang, Davis  and
      Nguyen, Toan Q.  and
      Kirchhoff, Katrin",
    editor = "Jurafsky, Dan  and
      Chai, Joyce  and
      Schluter, Natalie  and
      Tetreault, Joel",
    booktitle = "Proceedings of the 58th Annual Meeting of the Association for Computational Linguistics",
    month = jul,
    year = "2020",
    address = "Online",
    publisher = "Association for Computational Linguistics",
    url = "https://aclanthology.org/2020.acl-main.240/",
    doi = "10.18653/v1/2020.acl-main.240",
    pages = "2699--2712"
}

@inproceedings{gu-etal-2025-olmes,
    title = "{OLMES}: A Standard for Language Model Evaluations",
    author = "Gu, Yuling  and
      Tafjord, Oyvind  and
      Kuehl, Bailey  and
      Haddad, Dany  and
      Dodge, Jesse  and
      Hajishirzi, Hannaneh",
    editor = "Chiruzzo, Luis  and
      Ritter, Alan  and
      Wang, Lu",
    booktitle = "Findings of the Association for Computational Linguistics: NAACL 2025",
    month = apr,
    year = "2025",
    address = "Albuquerque, New Mexico",
    publisher = "Association for Computational Linguistics",
    url = "https://aclanthology.org/2025.findings-naacl.282/",
    doi = "10.18653/v1/2025.findings-naacl.282",
    pages = "5005--5033",
    ISBN = "979-8-89176-195-7"
}

@misc{clark2018thinksolvedquestionanswering,
      title={Think you have Solved Question Answering? {T}ry {ARC}, the {AI2} Reasoning Challenge}, 
      author={Peter Clark and Isaac Cowhey and Oren Etzioni and Tushar Khot and Ashish Sabharwal and Carissa Schoenick and Oyvind Tafjord},
      year={2018},
      eprint={1803.05457},
      archivePrefix={arXiv},
      primaryClass={cs.AI},
      url={https://arxiv.org/abs/1803.05457}, 
}

@article{warstadt-etal-2020-blimp-benchmark,
    title = "{BL}i{MP}: The Benchmark of Linguistic Minimal Pairs for {E}nglish",
    author = "Warstadt, Alex  and
      Parrish, Alicia  and
      Liu, Haokun  and
      Mohananey, Anhad  and
      Peng, Wei  and
      Wang, Sheng-Fu  and
      Bowman, Samuel R.",
    editor = "Johnson, Mark  and
      Roark, Brian  and
      Nenkova, Ani",
    journal = "Transactions of the Association for Computational Linguistics",
    volume = "8",
    year = "2020",
    address = "Cambridge, MA",
    publisher = "MIT Press",
    url = "https://aclanthology.org/2020.tacl-1.25/",
    doi = "10.1162/tacl_a_00321",
    pages = "377--392"
}

@inproceedings{talmor-etal-2019-commonsenseqa,
    title = "{C}ommonsense{QA}: A Question Answering Challenge Targeting Commonsense Knowledge",
    author = "Talmor, Alon  and
      Herzig, Jonathan  and
      Lourie, Nicholas  and
      Berant, Jonathan",
    editor = "Burstein, Jill  and
      Doran, Christy  and
      Solorio, Thamar",
    booktitle = "Proceedings of the 2019 Conference of the North {A}merican Chapter of the Association for Computational Linguistics: Human Language Technologies, Volume 1 (Long and Short Papers)",
    month = jun,
    year = "2019",
    address = "Minneapolis, Minnesota",
    publisher = "Association for Computational Linguistics",
    url = "https://aclanthology.org/N19-1421/",
    doi = "10.18653/v1/N19-1421",
    pages = "4149--4158"
}

@inproceedings{zellers-etal-2019-hellaswag,
    title = "{H}ella{S}wag: Can a Machine Really Finish Your Sentence?",
    author = "Zellers, Rowan  and
      Holtzman, Ari  and
      Bisk, Yonatan  and
      Farhadi, Ali  and
      Choi, Yejin",
    editor = "Korhonen, Anna  and
      Traum, David  and
      M{\`a}rquez, Llu{\'i}s",
    booktitle = "Proceedings of the 57th Annual Meeting of the Association for Computational Linguistics",
    month = jul,
    year = "2019",
    address = "Florence, Italy",
    publisher = "Association for Computational Linguistics",
    url = "https://aclanthology.org/P19-1472/",
    doi = "10.18653/v1/P19-1472",
    pages = "4791--4800"
}

@inproceedings{mihaylov-etal-2018-suit,
    title = "Can a Suit of Armor Conduct Electricity? {A} New Dataset for Open Book Question Answering",
    author = "Mihaylov, Todor  and
      Clark, Peter  and
      Khot, Tushar  and
      Sabharwal, Ashish",
    editor = "Riloff, Ellen  and
      Chiang, David  and
      Hockenmaier, Julia  and
      Tsujii, Jun{'}ichi",
    booktitle = "Proceedings of the 2018 Conference on Empirical Methods in Natural Language Processing",
    month = oct # "-" # nov,
    year = "2018",
    address = "Brussels, Belgium",
    publisher = "Association for Computational Linguistics",
    url = "https://aclanthology.org/D18-1260/",
    doi = "10.18653/v1/D18-1260",
    pages = "2381--2391"
}

@article{Bisk_2020, title={{PIQA}: Reasoning about Physical Commonsense in Natural Language}, volume={34}, url={https://ojs.aaai.org/index.php/AAAI/article/view/6239}, DOI={10.1609/aaai.v34i05.6239}, abstractNote={&lt;p&gt;To apply eyeshadow without a brush, should I use a &lt;em&gt;cotton swab or a toothpick&lt;/em&gt;? Questions requiring this kind of &lt;strong&gt;physical commonsense&lt;/strong&gt; pose a challenge to today’s natural language understanding systems. While recent pretrained models (such as BERT) have made progress on question answering over more &lt;em&gt;abstract&lt;/em&gt; domains – such as news articles and encyclopedia entries, where text is plentiful – in more &lt;em&gt;physical&lt;/em&gt; domains, text is inherently limited due to reporting bias. Can AI systems learn to reliably answer physical commonsense questions without experiencing the physical world?&lt;/p&gt;&lt;p&gt;In this paper, we introduce the task of physical commonsense reasoning and a corresponding benchmark dataset &lt;strong&gt;Physical Interaction: Question Answering&lt;/strong&gt; or &lt;strong&gt;PIQA&lt;/strong&gt;. Though humans find the dataset easy (95% accuracy), large pretrained models struggle (∼75%). We provide analysis about the dimensions of knowledge that existing models lack, which offers significant opportunities for future research.&lt;/p&gt;}, number={05}, journal={Proceedings of the AAAI Conference on Artificial Intelligence}, author={Bisk, Yonatan and Zellers, Rowan and Le bras, Ronan and Gao, Jianfeng and Choi, Yejin}, year={2020}, month={Apr.}, pages={7432-7439}}

@inproceedings{sap-etal-2019-social,
    title = "Social {IQ}a: Commonsense Reasoning about Social Interactions",
    author = "Sap, Maarten  and
      Rashkin, Hannah  and
      Chen, Derek  and
      Le Bras, Ronan  and
      Choi, Yejin",
    editor = "Inui, Kentaro  and
      Jiang, Jing  and
      Ng, Vincent  and
      Wan, Xiaojun",
    booktitle = "Proceedings of the 2019 Conference on Empirical Methods in Natural Language Processing and the 9th International Joint Conference on Natural Language Processing (EMNLP-IJCNLP)",
    month = nov,
    year = "2019",
    address = "Hong Kong, China",
    publisher = "Association for Computational Linguistics",
    url = "https://aclanthology.org/D19-1454/",
    doi = "10.18653/v1/D19-1454",
    pages = "4463--4473"
}

@inproceedings{10.5555/3666122.3667771,
author = {Lindner, David and Kram\'{a}r, J\'{a}nos and Farquhar, Sebastian and Rahtz, Matthew and McGrath, Thomas and Mikulik, Vladimir},
title = {Tracr: compiled transformers as a laboratory for interpretability},
year = {2023},
publisher = {Curran Associates Inc.},
address = {Red Hook, NY, USA},
booktitle = {Proceedings of the 37th International Conference on Neural Information Processing Systems},
articleno = {1649},
numpages = {24},
location = {New Orleans, LA, USA},
series = {NIPS '23},
url = {https://dl.acm.org/doi/10.5555/3666122.3667771}
}

@inproceedings{
xue2025anyorder,
title={Any-Order {GPT} as Masked Diffusion Model: Decoupling Formulation and Architecture},
author={Shuchen Xue and Tianyu Xie and Tianyang Hu and Zijin Feng and Jiacheng Sun and Kenji Kawaguchi and Zhenguo Li and Zhi-Ming Ma},
booktitle={ES-FoMo III: 3rd Workshop on Efficient Systems for Foundation Models},
year={2025},
url={https://openreview.net/forum?id=KbRxn8fzrY}
}

@inproceedings{10.5555/3666122.3667859,
author = {Wu, Tong and Fan, Zhihao and Liu, Xiao and Zheng, Hai-Tao and Gong, Yeyun and Shen, Yelong and Jiao, Jian and Li, Juntao and Wei, Zhongyu and Guo, Jian and Duan, Nan and Chen, Weizhu},
title = {{AR-D}IFFUSION: auto-regressive diffusion model for text generation},
year = {2023},
publisher = {Curran Associates Inc.},
address = {Red Hook, NY, USA},
booktitle = {Proceedings of the 37th International Conference on Neural Information Processing Systems},
articleno = {1737},
numpages = {18},
location = {New Orleans, LA, USA},
series = {NIPS '23},
url = {https://dl.acm.org/doi/abs/10.5555/3666122.3667859}
}

@inproceedings{
arriola2025block,
title={Block Diffusion: Interpolating Between Autoregressive and Diffusion Language Models},
author={Marianne Arriola and Subham Sekhar Sahoo and Aaron Gokaslan and Zhihan Yang and Zhixuan Qi and Jiaqi Han and Justin T Chiu and Volodymyr Kuleshov},
booktitle={The Thirteenth International Conference on Learning Representations},
year={2025},
url={https://openreview.net/forum?id=tyEyYT267x}
}

@InProceedings{pmlr-v139-weiss21a,
  title = 	 {Thinking Like Transformers},
  author =       {Weiss, Gail and Goldberg, Yoav and Yahav, Eran},
  booktitle = 	 {Proceedings of the 38th International Conference on Machine Learning},
  pages = 	 {11080--11090},
  year = 	 {2021},
  editor = 	 {Meila, Marina and Zhang, Tong},
  volume = 	 {139},
  series = 	 {Proceedings of Machine Learning Research},
  month = 	 {18--24 Jul},
  publisher =    {PMLR},
  url = 	 {https://proceedings.mlr.press/v139/weiss21a.html}
}

@inproceedings{katz-etal-2025-segment,
    title = "Segment-Based Attention Masking for {GPT}s",
    author = "Katz, Shahar  and
      Ringel, Liran  and
      Romano, Yaniv  and
      Wolf, Lior",
    editor = "Che, Wanxiang  and
      Nabende, Joyce  and
      Shutova, Ekaterina  and
      Pilehvar, Mohammad Taher",
    booktitle = "Proceedings of the 63rd Annual Meeting of the Association for Computational Linguistics (Volume 1: Long Papers)",
    month = jul,
    year = "2025",
    address = "Vienna, Austria",
    publisher = "Association for Computational Linguistics",
    url = "https://aclanthology.org/2025.acl-long.947/",
    doi = "10.18653/v1/2025.acl-long.947",
    pages = "19308--19322",
    ISBN = "979-8-89176-251-0"
}

@inproceedings{yu-etal-2024-antlm,
    title = "{A}nt{LM}: Bridging Causal and Masked Language Models",
    author = "Yu, Xinru  and
      Guo, Bin  and
      Luo, Shiwei  and
      Wang, Jie  and
      Ji, Tao  and
      Wu, Yuanbin",
    editor = "Hu, Michael Y.  and
      Mueller, Aaron  and
      Ross, Candace  and
      Williams, Adina  and
      Linzen, Tal  and
      Zhuang, Chengxu  and
      Choshen, Leshem  and
      Cotterell, Ryan  and
      Warstadt, Alex  and
      Wilcox, Ethan Gotlieb",
    booktitle = "The 2nd BabyLM Challenge at the 28th Conference on Computational Natural Language Learning",
    month = nov,
    year = "2024",
    address = "Miami, FL, USA",
    publisher = "Association for Computational Linguistics",
    url = "https://aclanthology.org/2024.conll-babylm.29/",
    pages = "324--331"
}

@inproceedings{holtzman-etal-2021-surface,
    title = "Surface Form Competition: Why the Highest Probability Answer Isn{'}t Always Right",
    author = "Holtzman, Ari  and
      West, Peter  and
      Shwartz, Vered  and
      Choi, Yejin  and
      Zettlemoyer, Luke",
    editor = "Moens, Marie-Francine  and
      Huang, Xuanjing  and
      Specia, Lucia  and
      Yih, Scott Wen-tau",
    booktitle = "Proceedings of the 2021 Conference on Empirical Methods in Natural Language Processing",
    month = nov,
    year = "2021",
    address = "Online and Punta Cana, Dominican Republic",
    publisher = "Association for Computational Linguistics",
    url = "https://aclanthology.org/2021.emnlp-main.564/",
    doi = "10.18653/v1/2021.emnlp-main.564",
    pages = "7038--7051"
}

@inbook{10.5555/3454287.3455008,
author = {Paszke, Adam and Gross, Sam and Massa, Francisco and Lerer, Adam and Bradbury, James and Chanan, Gregory and Killeen, Trevor and Lin, Zeming and Gimelshein, Natalia and Antiga, Luca and Desmaison, Alban and K\"{o}pf, Andreas and Yang, Edward and DeVito, Zach and Raison, Martin and Tejani, Alykhan and Chilamkurthy, Sasank and Steiner, Benoit and Fang, Lu and Bai, Junjie and Chintala, Soumith},
title = {PyTorch: an imperative style, high-performance deep learning library},
year = {2019},
publisher = {Curran Associates Inc.},
address = {Red Hook, NY, USA},
booktitle = {Proceedings of the 33rd International Conference on Neural Information Processing Systems},
articleno = {721},
numpages = {12},
url = {https://dl.acm.org/doi/10.5555/3454287.3455008}
}

@InProceedings{pmlr-v162-wang22u,
  title = 	 {What Language Model Architecture and Pretraining Objective Works Best for Zero-Shot Generalization?},
  author =       {Wang, Thomas and Roberts, Adam and Hesslow, Daniel and Scao, Teven Le and Chung, Hyung Won and Beltagy, Iz and Launay, Julien and Raffel, Colin},
  booktitle = 	 {Proceedings of the 39th International Conference on Machine Learning},
  pages = 	 {22964--22984},
  year = 	 {2022},
  editor = 	 {Chaudhuri, Kamalika and Jegelka, Stefanie and Song, Le and Szepesvari, Csaba and Niu, Gang and Sabato, Sivan},
  volume = 	 {162},
  series = 	 {Proceedings of Machine Learning Research},
  month = 	 {17--23 Jul},
  publisher =    {PMLR},
  url = 	 {https://proceedings.mlr.press/v162/wang22u.html}
}

@misc{open-llm-leaderboard-v2,
  author = {Clémentine Fourrier and Nathan Habib and Alina Lozovskaya and Konrad Szafer and Thomas Wolf},
  title = {Open {LLM} Leaderboard v2},
  year = {2024},
  publisher = {Hugging Face},
  url={https://huggingface.co/spaces/open-llm-leaderboard/open_llm_leaderboard}
}
